\documentclass{tilde}
\usepackage{microtype}

\newcommand{\E}{\mathbb{E}}

\newcommand{\dd}{\mathrm{d}}
\newcommand{\softmax}{\mathrm{softmax}}

\newcommand{\brho}{\bm{\rho}}
\newcommand{\bq}{\bm{q}}
\newcommand{\bk}{\bm{k}}
\newcommand{\bv}{\bm{v}}
\newcommand{\bx}{\bm{x}}
\newcommand{\bz}{\bm{z}}
\newcommand{\bo}{\bm{o}}

\DeclareMathOperator*{\argmin}{arg\,min}

\DeclareMathOperator{\rowmax}{rowmax}
\DeclareMathOperator{\rowsum}{rowsum}

\newcommand{\papertitle}{Parallax: Parameterized Local Linear Attention for Language Modeling}
\newcommand{\paperdate}{May 2026}
\newcommand{\contactemail}{yifeizuo2029@u.northwestern.edu}
\newcommand{\codeurl}{https://github.com/yifei-zuo/Parallax}

\newcommand{\tildelogo}{%
  \includegraphics[height=1.1em]{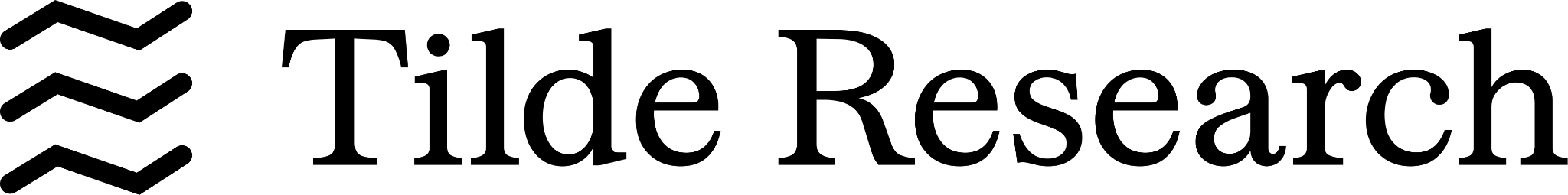}%
}
\newcommand{\northwesternlogo}{%
  \includegraphics[height=2.2em]{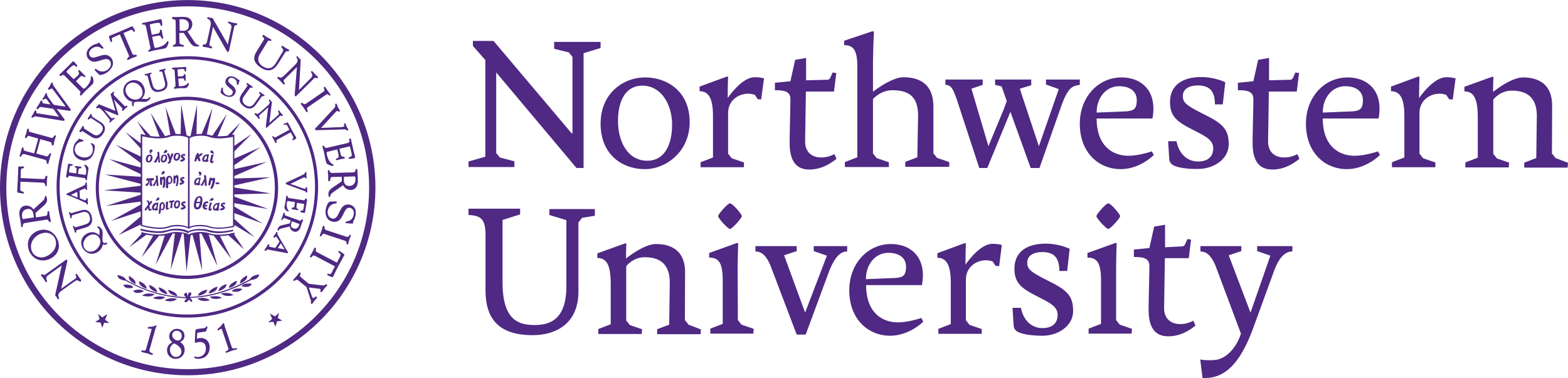}%
}
\newcommand{\uwlogo}{%
  \includegraphics[height=1.85em]{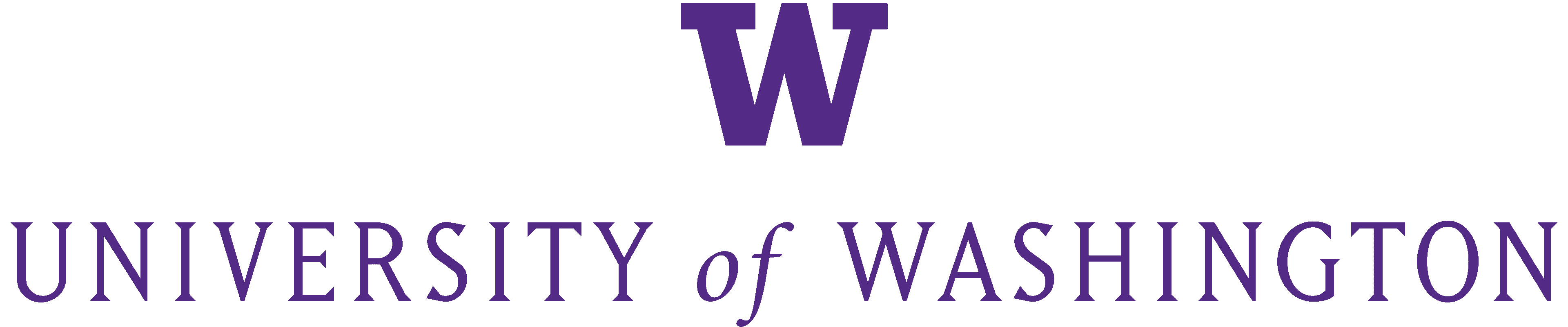}%
}

\begin{document}


\begin{herobox}
{\setlength{\parskip}{0pt}%

{\titlefont\papertitle\par}

\vskip 0.2cm

{\authorfont
Yifei Zuo\textsuperscript{1},
Dhruv Pai\textsuperscript{2},
Zhichen Zeng\textsuperscript{3},
Alec Dewulf\textsuperscript{2},
Shuming Hu\textsuperscript{2},
Zhaoran Wang\textsuperscript{1}\par}

\vskip 0.2cm

{\normalsize
\textsuperscript{1}Northwestern University,
\textsuperscript{2}Tilde Research,
\textsuperscript{3}University of Washington\par}

\vskip 0.5cm

Large Language Models (LLMs) have become the central paradigm in artificial intelligence, yet the core computational primitive of attention has remained structurally unchanged. Local Linear Attention (LLA) is an attention mechanism derived from nonparametric statistics in the test-time regression framework. In contrast to prior research on efficient attention variants, LLA upgrades the local constant estimate in softmax attention to a local linear estimate, yielding provably superior bias-variance tradeoffs for associative memory. However, LLA has not been scaled in LLM pretraining due to computational and numerical stability concerns.
We introduce \textbf{Parallax}, a \textbf{para}meterized \textbf{L}ocal \textbf{L}inear \textbf{A}ttention that is scalable for LLMs. Parallax eliminates the numerical solver in LLA and learns an extra query-like projector that probes the KV covariance.
We place Parallax within a family of attention mechanisms connected by the bandwidth, the probe construction and the affine structure.
We propose a hardware-aware algorithm that increases the arithmetic intensity over FlashAttention, shifting attention into a more compute bound regime.
Our prototype decode kernel matches or outperforms FlashAttention 2/3 across diverse batch sizes and context lengths.
We pretrain Parallax at 0.6B and 1.7B scales and find consistent perplexity improvements throughout pretraining with gains that transfer to downstream benchmarks.
The advantage persists under both parameter-matched and compute-matched controls, demonstrating a Pareto improvement.
We perform careful pretraining ablations and identify a novel phenomenon whereby Muon unlocks the capacity of Parallax. To our knowledge, this is the first empirical demonstration of strong architecture-optimizer codesign for attention mechanisms in the architecture research literature.\par

\vskip 0.5cm

{\metadatafont
\begin{tabular}{@{}ll}
\textbf{Date:} & \paperdate \\
\textbf{Code:} & \href{\codeurl}{\codeurl}\\
\textbf{Correspondence:} & \href{mailto:\contactemail}{\contactemail}\\
 & \href{mailto:dhruv@tilderesearch.com}{dhruv@tilderesearch.com}\\
 & \href{mailto:zhaoranwang@gmail.com}{zhaoranwang@gmail.com}
\end{tabular}
}

\vspace{-6.5em}
\begin{flushright}
\northwesternlogo\\[0.4em]
\uwlogo\\[0.4em]
\tildelogo
\end{flushright}

}
\end{herobox}

\vspace{2.0em}


\section{Introduction}

Large Language Models (LLMs) have become the central paradigm in artificial intelligence, powering advances in mathematical reasoning, code generation, multimodal processing and scientific discovery.
Throughout the rapid progress of LLMs, Softmax Attention \citep{vaswani2023attentionneed} has remained largely unchanged as the backbone of the Transformer architectures.
A substantial body of work has sought efficient alternatives to Softmax Attention for long-context generation.
For example, Linear Attention such as DeltaNet \citep{yang2025gateddeltanetworksimproving,yang2025parallelizinglineartransformersdelta,kimiteam2025kimilinearexpressiveefficient}, and State Space Models (SSMs) such as Mamba \citep{gu2024mambalineartimesequencemodeling} maintain constant-size recurrent states and achieve subquadratic complexity.
Despite the efficiency gains, such models consistently underperform Softmax Attention on in-context information retrieval \citep{arora2023zoologymeasuringimprovingrecall,pmlr-v267-bick25a,10.5555/3692070.3692933}, suggesting the underlying trade-off behind these design choices.
The test-time regression framework \citep{wang2025testtimeregressionunifyingframework} unifies these attention mechanism designs by interpreting them as in-context regression solvers.
Local Linear Attention (LLA) \citep{zuo2025locallinear} sharpens this perspective beyond Linear Attention by connecting the bias-variance theory with associative memory capacity, and shows that replacing the local constant estimator of Softmax Attention with a local linear estimator yields a strictly richer and more powerful predictor.

Although LLA has theoretical advantages and strong results on synthetic tasks, it has not yet been shown effective for large scale LLM pretraining.
Specifically, the per-token conjugate gradient solve introduces both computation and I/O overhead and numerical sensitivity that are difficult to manage at scale.
To bridge this gap, we propose Parallax, a parameterized LLA that preserves the local linear principle while being more efficient, scalable, and simpler to implement.
It accepts an extra \(\bm R\) matrix alongside the standard \(\bm Q,\bm K\) and \(\bm V\) matrices, and learns to probe the KV covariance to improve the prediction.
Notably, we demonstrate an \emph{optimizer-architecture} interaction that was not previously recognized, whereby the correction branch in Parallax depends strongly on the optimizer geometry.
Empirically, we find that the Muon optimizer \citep{jordan2024muon} is crucial for Parallax to demonstrate consistent improvements over Softmax Attention.

\paragraph{Contributions.}
To summarize, our contributions are:
\begin{enumerate}
    \item \textbf{Architecture. }We identify the key challenges in scaling LLA to pretraining and derive Parallax to tackle these issues.
    We provide a unified interpretation that connects nonparametric attention mechanisms to their parametric counterparts, clarifying their design tradeoffs and complexity.
    \item \textbf{Efficiency. }We analyze the I/O and compute complexity of Parallax and develop a hardware-aware streaming algorithm. Our custom decode kernel matches or outperforms FlashAttention 2/3 across a wide range of batch sizes and context lengths.
    \item \textbf{Experiment. }We validate Parallax on synthetic tasks and on LLM pretraining at 0.6B and 1.7B scales, where it consistently improves perplexity and downstream accuracy over Softmax Attention.
    The improvement persists under both parameter-matched and compute-matched controls.
    We further characterize a strong optimizer-architecture interaction where Parallax shows substantial advantage under Muon, while the two are comparable under AdamW.
\end{enumerate}

\subsection{Related Work.}
\paragraph{Efficient Attention Mechanism.}
The quadratic computation and expensive I/O in Softmax Attention~\citep{vaswani2023attentionneed} has motivated a broad search for efficient alternatives.
Linear Attention~\citep{katharopoulos2020transformers} removes the softmax operation, enabling recurrent inference with a constant-size state.
Subsequent work has enriched this family through Retention~\citep{sun2023retentivenetworksuccessortransformer}, Gating~\citep{yang2024gatedlinearattentiontransformers}, Delta-Rules~\citep{yang2025parallelizinglineartransformersdelta, yang2025gateddeltanetworksimproving} and Householder products~\citep{siems2025deltaproductimprovingstatetrackinglinear}.
Similarly, SSMs such as Mamba~\citep{gu2024mambalineartimesequencemodeling,dao2024transformersssmsgeneralizedmodels,lahoti2026mamba3} aim to parameterize linear recurrences with structured matrices for long-horizon recall \citep{gu2022efficiently,gu2022trainhippostatespace,poli2023hyenahierarchylargerconvolutional}.
FlashAttention~\citep{dao2022flashattentionfastmemoryefficientexact,dao2023flashattention2fasterattentionbetter,shah2024flashattention3} explores hardware-aware algorithm innovations, while keeping the underlying mechanism unchanged.
Sparse Softmax Attentions~\citep{yuan2025native, gao2024seerattention, lu2025moba, xiao2025statistics} and GQA, MLA \citep{ainslie2023gqa,deepseekai2024deepseekv2strongeconomicalefficient} further incorporate the I/O aware design, making the attention more efficient in practice.

\paragraph{Attention as test-time learner.}
A growing body of work shows that attention mechanisms implicitly implement optimization steps to perform in-context learning
\citep{garg2023transformerslearnincontextcase,akyurek2022learning,vonoswald2023transformerslearnincontextgradient,kirsch2024generalpurposeincontextlearningmetalearning,zhang2024trained,mahankali2024stepgradientdescentprovably,ahn2023transformerslearnimplementpreconditioned,dai2023gptlearnincontextlanguage}.
This perspective has motivated a series of attention variants designed around explicit test-time objectives, including Titans
\citep{behrouz2025titans}, MIRAS \citep{behrouz2026itsallconnected}, MesaNet~\citep{vonoswald2025mesanetsequencemodelinglocally}, and TTT~\citep{pmlr-v267-sun25h}. The test-time regression framework
\citep{wang2025testtimeregressionunifyingframework} unifies these designs by interpreting them as in-context regression solvers, from which LLA \citep{zuo2025locallinear} is derived.

\paragraph{Optimizers for LLMs.}
Adam(W) \citep{kingma2015adam,loshchilov2019decoupled} has long been the de facto choice of optimizer for all stages of the training pipeline. Subsequent work proposes optimizers that use more expressive curvature approximations~\citep{gupta2018shampoo, martens2015optimizing, vyas2024soap} but these methods have yet to gain traction, partially due to increased memory and compute costs.
Recently, Muon~\citep{jordan2024muon} has become a popular alternative to Adam(W) for optimizing matrix parameters in the hidden layers. Moonlight~\citep{liu2025muonscalable} adds RMS-matched updates and weight decay, making Muon more scalable. Dion~\citep{ahn2025dion} explores cheaper ways to orthogonalize the gradient, and methods of reducing Muon's communication cost in distributed settings. Further work explores more precise Newton-Schulz methods~\citep{amsel2025polar, grishina2025accelerating}, which have been shown to improve the downstream performance of Muon. These efforts have culminated in Muon's application to training frontier-scale models~\citep{team2026kimi, zeng2026glm}. 

\subsection{Notation}
For a matrix \(\bm X\), we denote \(\|\bm X\|_F\) the Frobenius norm, \(\|\bm X\|_2\) the spectral norm, and use \(\odot\) to denote the Hadamard product between matrices.
We use \(\texttt{srank}(\bm X)\) to denote the stable rank of \(\bm X\), defined as \(\|\bm X\|_F^2 / \|\bm X\|_2^2\).
For a vector \(\bm x\), we use \(\|\bm x\|\) to denote its Euclidean norm.
To distinguish them from variables in the main text, matrix and vector variables in algorithm descriptions are denoted by \(\mathbf{X}\) and \(\mathbf{x}\), respectively.
\section{Preliminary}
\label{sec:prelim}
\subsection{Local Linear Attention}
\paragraph{Test Time Regression. }
The test-time regression framework~\citep{wang2025testtimeregressionunifyingframework} interprets the attention mechanism as a regression solver over the KV pairs $\mathcal D_i = \{(\bk_j, \bv_j)\}_{j\le i}$.
The key vectors are treated as the training data points and value vectors are the labels.
The attention function learns to predict on the test data point \(\bq_i\).
Specifically, denote \(\mathcal F(\bq_i)\) the hypothesis space, \(\Omega(f)\) the regularization and \(w_{ij}\) the weighting factor.
The objective can be generally formulated as
\begin{align}
    \hat f(\bq_i) = \argmin_{f\in\mathcal F(\bq_i)}\Bigl\{\sum_{j\le i} w_{ij} \|f(\bk_j) - \bv_j\|^2 + {\Omega}(f)\Bigr\}.
\end{align}
Different attention designs correspond to different specifications of \emph{hypothesis spaces}, \emph{objective functions} and \emph{optimization methods}.
For example, Linear Attention family corresponds to the parametric linear estimators with \(\mathcal F = \{\bm W\bx+\bm b\}\) and context-independent weighting.
MesaNet \citep{vonoswald2025mesanetsequencemodelinglocally} chooses \(\Omega(f) = \lambda \|\bm W\|^2_F\) and solves the optimal ridge regression, while DeltaNet takes one step of stochastic gradient descent on the current KV pair without regularization.
In contrast to the parametric approaches, Softmax Attention is nonparametric.
It employs the Nadaraya-Watson (NW) estimator \citep{Nadaraya1964,Watson1964,Bierens1988} with kernel \(w_{ij}=\exp(\bq_i^\top \bk_j / h)\).
The hypothesis space simply contains constant functions \(\mathcal F(\bq_i)=\{\bm c\}\) built for each query.

These design choices, particularly the choice of hypothesis space, fundamentally impact the associative memory capacity of each mechanism.
Linear Attention suffers from the irreducible misspecification error, while Softmax Attention suffers from the boundary bias, which can be resolved by upgrading its constant function class to linear function class.
\citet{zuo2025locallinear} prove that by doing so the model can achieve strictly smaller integrated MSE.
We provide a brief review of the main results in Theorem~\ref{thm:separation}.

\begin{theorem}[Bias-variance separation \citep{zuo2025locallinear}]
\label{thm:separation}
Let $(X_i, Y_i)_{i=1}^n$ be i.i.d.\ with $X_i \in \mathbb R^d$ supported on a bounded $C^2$ domain $D$ and $Y_i = f(X_i) + \varepsilon_i \in \mathbb R^{d_y}$, $\E[\varepsilon_i \mid X_i] = 0$.
Let $\hat f_{\mathrm{GL}}$, $\hat f_{\mathrm{NW}}$, and $\hat f_{\mathrm{LL}}$ denote the Global Linear, Nadaraya--Watson (Local Constant), and Local Linear estimators with optimal bandwidths, respectively.
Under Assumptions~\ref{asm:domain}--\ref{asm:boundary}, denote $\mathcal{R}(\hat f) := \E\!\int_D \|\hat f(x) - f(x)\|^2\, dx$ the integrated mean squared error, then
\begin{equation}
    \mathcal{R}(\hat f_{\mathrm{GL}})
    \;\gg\;
    \mathcal{R}(\hat f_{\mathrm{NW}})
    \;\gg\;
    \mathcal{R}(\hat f_{\mathrm{LL}}).
\end{equation}
The lower bound for $\hat f_{\mathrm{GL}}$ holds whenever $f$ is not globally affine.
The lower bound for $\hat f_{\mathrm{NW}}$ holds whenever $f$ has sufficiently large normal gradient along $\partial D$ (Assumption~\ref{asm:boundary}).
\end{theorem}

\paragraph{Local Linear Attention. }LLA fits a local linear estimator \(f\in \mathcal F(\bq_i)=\{\bm b+\bm W(\bx-\bq_i)\}\) equipped with kernel weight \(w_{ij}=\exp(\bq_i^\top \bk_j/h)\) and ridge regularization \(\lambda\|\bm W\|_F^2\). Let \(\bz_{ij}=\bk_j-\bq_i\), \(\omega_i=\sum_{j\le i}w_{ij}\), \(\bm\mu_i=\sum_{j\le i}w_{ij}\bz_{ij}\), and \(\bm\Sigma_i=\sum_{j\le i}w_{ij}\bz_{ij}\bz_{ij}^{\top}+\lambda \bm I\),
LLA is the prediction of local linear estimator at \(\bq_i\):
\begin{align}
    \bo_i^{\textsf{LLA}}=\hat f_{\text{LL}}(\bq_i)=\sum_{j\le i}\frac{w_{ij}(1-\bz_{ij}^{\top}\brho_i^\star)}{\omega_i-\bm\mu_i^{\top}\brho_i^\star} \bv_j,\quad \brho_i^\star=\bm\Sigma_i^{-1}\bm\mu_i.
    \label{eq:lla_forward}
\end{align}
Intuitively, LLA is a query-centered second-order correction to Softmax Attention leveraging the geometry of key vectors around the query.
It provides a better prediction when the keys are not uniformly distributed under the softmax geometry.
LLA can also be interpreted as constructing query-dependent states through the kernel, in contrast to the global states in MesaNet.
As shown in Figure~\ref{fig:connections}, LLA can degenerate to both mechanisms by tuning \(\lambda\) and \(h\):
\begin{align}
    \bo_i^{\textsf{LLA}} \xrightarrow{\lambda\to\infty}\; \bo_i^{\textsf{SA}}& = \softmax(\bq_i^\top \bk_j / h)\,\bv_j,\\
    \bo_i^{\textsf{LLA}} \xrightarrow[\text{drop intercept}]{h\to\infty} \bo_i^{\textsf{Mesa}} = &\biggl(\sum_{j\le i} \bv_j\bk_j^\top\biggr) \biggl(\sum_{j\le i} \bk_j\bk_j^\top + \lambda \bm I\biggr)^{-1} \bq_i.
    \label{eq:lla_degenerate}
\end{align}
\paragraph{Challenges for LLM training with LLA. }
Despite its appealing theoretical properties and empirical advantages in synthetic tasks, LLA faces several challenges when scaled up to realistic language model training.
In particular, the exact LLA forward requires solving a linear system \(\bm\Sigma_i \bx = \bm\mu_i\) for every query with a parallel conjugate gradient (CG) solver. It introduces several practical issues:
\begin{itemize}
    \item \emph{Intensive I/O. }The CG iteration requires \(TLd\) memory access in the forward pass, dominating the \(2Ld\) memory access of Softmax Attention.
    \(1\le T \le d\) is the iteration number.
    \item \emph{Regularization-expressiveness tradeoff. }Large \(\lambda\) ensures \(\bm\Sigma_i \succ 0\) but drives \(\brho_i^\star \to \bm 0\), making LLA degenerate to Softmax Attention; small \(\lambda\) enables more expressiveness but risks ill-conditioning and instability.
    We find it nontrivial to balance the tradeoff in practical pretraining settings.
    \item \emph{Low-precision incompatibility. }The stability of CG is sensitive to the precision format,
    while modern hardware and computation primitives are increasingly shaped around reduced precision.
\end{itemize}

\subsection{Muon Optimizer}
Muon is a novel optimizer for matrix parameters in the hidden layers. For a weight matrix \(\bm W_t\in\mathbb R^{m\times n}\) with gradient \(\bm G_t=\nabla_{\bm W} L(\bm W_t)\), Muon maintains a momentum buffer \(\bm B_t= \beta \bm B_{t-1}+\bm G_t\) with \(\bm B_0=0\).
Letting the singular value decomposition (SVD) of \(\bm B_t\) be \(\bm B_t=\bm U_t \bm S_t \bm V_t^\top\), Muon forms the polar factor \(\mathrm{polar}(\bm B_t)=\bm U_t\bm V_t^\top\), which is the nearest semi-orthogonal matrix in the Frobenius norm, and updates $\bm W_t$ according to:
\begin{align}
    \bm W_{t+1}=\bm W_t-\eta_t \bm U_t\bm V_t^\top.
    \label{eq:muon_update}
\end{align}

Note for clarity, weight decay is omitted.

Computing $\bm U_t$ and $\bm V_t$ via SVD is prohibitively expensive. In practice, \(\bm U_t\bm V_t^\top\) is approximated by Newton--Schulz iterations with precisely tuned matrix polynomials. These methods avoid a full SVD and can converge to a precise estimate of the polar factor in just a small number of steps~\citep{jordan2024muon,liu2025muonscalable}. This approach has the added benefit of exploiting fast GEMM subroutines on GPUs, making Muon hardware-aligned and feasible to use at scale. 

\citet{bernstein2024oldoptimizernewnorm} interpret the Muon update as steepest descent under the operator norm \(\|\cdot\|_{\ell_2 \to \ell_2}\), which for matrices coincides with the spectral norm. The polar factor has all singular values equal to one, and so Muon's updates are guaranteed to have condition number of exactly one. Previous work has shown this strong conditioning of updates results in the underlying weight matrices themselves becoming better conditioned~\citep{boreiko2025towards,wang2026muon}. By contrast, matrices trained with AdamW, can exhibit spectral collapse \citet{arefin2026learning} whereby their effective rank shrinks rapidly over training. SignSGD and Adam can be interpreted as steepest descent under \(\|\cdot\|_{\ell_1\to\ell_\infty}\) geometry instead.


\section{Parallax Mechanism}

\subsection{Parameterized Local Linear Attention}
We first reformulate LLA as applying an additive correction to Softmax Attention with a projected KV covariance.
Write \(\bm p_i=\softmax(\bm K_i\bq_i/h)\), \(t_{ij} = \brho_i^{\star\top} \bz_{ij}\) and \(\bar t_{i} = \E_{\bm p_i}[t_{ij}]\).
The equation~\eqref{eq:lla_forward} can be rewritten as
\begin{align}
    \bo^{\textsf{LLA}}_i = \sum_{j\le i}\frac{w_{ij}(1-t_{ij})}{\sum_{j\le i}w_{ij}(1-t_{ij})}\bv_j = \bo^{\textsf{SA}}_i - (1+\eta_i)\bm \Sigma_{KV}^{(i)} \brho_i^\star,\quad \brho_i^\star = \bm\Sigma_i^{-1}\bm\mu_i,
    \label{eq:parameterized_lla}
\end{align}
where \(\bm\Sigma_{KV}^{(i)} = \E_{\bm p_i}[(\bv_j - \bar\bv_i)(\bk_j - \bar\bk_i)^\top]\), \(\bar \bv_i = \E_{\bm p_i}[\bv_j]\), \(\bar\bk_i = \E_{\bm p_i}[\bk_j]\) and \(\eta_i = \bar t_i / (1-\bar t_i)\) is the \emph{boundary amplification} factor.
By Proposition~\ref{prop:nonnegative_eta}, \(\eta_i\) is non-negative and quantifies the Mahalanobis distance from the query to the key center under \(\bm p_i\).
Intuitively, if \(\eta_i\approx 0\), the query is close to the weighted key center and the correction becomes pure covariance;
if \(\eta_i\gg 1\), the query is close to the weighted key boundary and the correction is amplified to compensate for the boundary bias.
\begin{proposition}[Boundary amplification is non-negative]
    Denote \(\bar z_{i} = \E_{\bm p_i}[\bz_{ij}]\) and \(\bm A_i = \omega_i\mathrm{Var}_{\bm p_i}(\bz_{ij}) + \lambda \bm I\succ 0\). If \(\brho_i^\star = \bm\Sigma_i^{-1}\bm\mu_i\), then \(\bar t_i = \E_{\bm p_i}[\brho_i^{\star\top} \bz_{ij}] \in [0, 1)\), \(\eta_i = \omega_i\bar\bz_i^\top \bm A_i^{-1}\bar \bz_i \ge 0\) where \(\bar z_i = \E_{\bm p_i}[\bz_{ij}]\).
    The proof is provided in Appendix~\ref{app:derivation_nonnegative_eta}.
\label{prop:nonnegative_eta}
\end{proposition}
\paragraph{Parallax formulation. }
Building on the above reformulation, Parallax eliminates the per-query solve of \(\brho_i^\star\) by learning a direct mapping from the layer input.
Let \(\bx_i\) be the input to the layer, we parameterize \(\brho_i = \bm W_R \bx_i\) where \(\bm W_R \in \mathbb R^{d_{\text{qk}}\times d}\) is a learnable projection matrix.
Parallax additionally sets \(\eta_i = 0\) to remove the boundary amplification, yielding the forward equation
\begin{align}
    \bo^{\textsf{PLX}}_i =\sum_{j\le i}\frac{w_{ij}(1-t_{ij}+\bar t_i)}{\sum_{j'\le i}w_{ij'}(1-t_{ij'}+\bar t_i)}\bv_j= \bo^{\textsf{SA}}_i - \bm \Sigma_{KV}^{(i)} \brho_i,\quad \brho_i = \bm W_R \bx_i.
    \label{eq:parallax_forward}
\end{align}
Removing \(\eta_i\) is necessary because the parameterized \(\brho_i = \bm W_R \bx_i\) is no longer the constrained solution of the exact LLA.
Once that structure is broken, the mean score \(\bar t_i\) no longer admits its geometric interpretation and can take unbounded values. The scaling factor \(1/(1-\bar t_i)\) can therefore diverge as \(\bar t_i\to 1\) or flip sign when \(\bar t_i > 1\), causing training instability.
Equivalently, setting \(\eta_i=0\) corresponds to replacing \(t_{ij}\) with the centered statistics \(t_{ij} - \bar t_i\) in the scoring form of equation~\eqref{eq:parameterized_lla}.
The denominator in the scoring form of equation~\eqref{eq:parallax_forward} reduces to \(\omega_i\), which is bounded away from zero with the safe softmax implementation \citep{milakov2018onlinenormalizercalculationsoftmax}.

\begin{figure}
    \centering
    \includegraphics[width=0.75\textwidth]{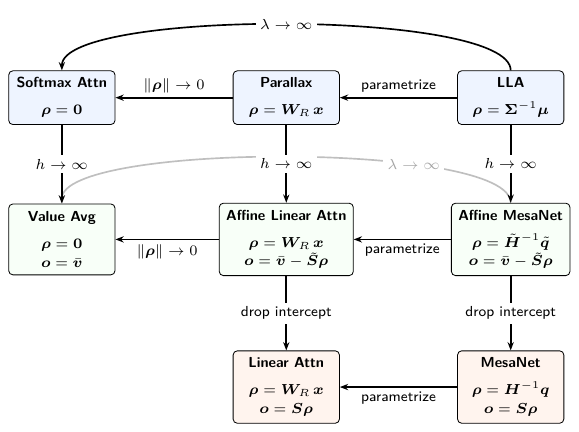}
    \caption{
        A family of attention mechanisms and their relationship. Rows are softmax weighted (top), uniform weighted with intercept (middle), and uniform weighted without intercept (bottom);
        Columns differs in how the probe \(\brho_i\) is obtained: zero (left), parametric (middle) and solved (right).
    }
    \label{fig:connections}
\end{figure}

\subsection{Connection to Other Attention Mechanisms}
\label{sec:connections}
To position Parallax relative to other attention mechanisms, we examine the wide bandwidth limit $h \to \infty$ and strong regularization limit $\lambda\to\infty$ (equivalently $\|\brho_i\|\to 0$).
We use the uniform-weighted running averages and second moments
\begin{align}
    \bar\bv_i = \tfrac{1}{i}\sum\nolimits_{j\le i} \bv_j,
    \quad
    \bar\bk_i = \tfrac{1}{i}\sum\nolimits_{j\le i} \bk_j,
    \quad
    \bm S_i = \tfrac{1}{i}\sum\nolimits_{j\le i} \bv_j \bk_j^\top,
    \quad
    \bm H_i = \tfrac{1}{i}\sum\nolimits_{j\le i} \bk_j \bk_j^\top + \lambda \bm I,
\end{align}together with the centered correspondences
\begin{align}
    \tilde \bq_i = \bar \bk_i - \bq_i,\qquad \tilde{\bm S}_i = \bm{S}_i - \bar \bv_i \bar\bk_i^\top,\qquad \tilde{\bm H}_i = \bm{H}_i - \bar\bk_i \bar\bk_i^\top.
\end{align}The connections are summarized in Figure~\ref{fig:connections}.
\paragraph{Wide bandwidth limit. }As \(h\to\infty\), the kernel weight $w_{ij} = \exp(\bq_i^\top \bk_j / h) \to 1$ uniformly in $j$, so the softmax weight degenerates to the uniform distribution $p_{ij} \to 1/i$, and $\omega_i \to i$.
Local softmax-weighted statistics become global running averages, and the three nonparametric mechanisms reduce to the \emph{affine} variants with recurrent states given by the uniform averages,
\begin{align}
\bo_i^{\textsf{SA}}  &\xrightarrow{h \to \infty} \bar\bv_i, &\brho_i &= \bm{0} &&\text{(Value Averaging)}, \\
\bo_i^{\textsf{PLX}} &\xrightarrow{h \to \infty} \bar\bv_i - \bm{\tilde S}_i \brho_i, &\brho_i &= \bm{W}_R \bx_i &&\text{(Affine Linear Attention)}, \\
\bo_i^{\textsf{LLA}} &\xrightarrow{h \to \infty} \bar\bv_i - \bm{\tilde S}_i \brho_i^\star, &\brho_i^\star &= \tilde{\bm{H}}_i^{-1} \tilde \bq_i &&\text{(Affine MesaNet)}.
\end{align}
All three share the output template and differ only in the probe. This template is the empirical OLS regression of $\bv$ on $\bk$ {with intercept} $\bar\bv_i$, evaluated at the query. The three mechanisms correspond to evaluating it through a zero, learnable, and fully solved $\brho_i$ respectively.
We refer to the corresponding attention mechanisms as Value Averaging, Affine Linear Attention and Affine MesaNet.

The standard forms of Linear Attention and MesaNet drop the intercept.
Algebraically, this is the same as setting $\bar\bv_i = \bar\bk_i = \bm{0}$ in the affine forms above, which collapses the centered moments to their raw counterparts.
The framework also clarifies the dual role of the query across the family. In nonparametric mechanisms, \(\bq_i\) shapes the kernel weights, defining {where} attention concentrates.
In the Linear Attention family, what is conventionally considered the query is in fact the probe \(\brho_i\), a directional readout from the recurrent state that can be completely determined by other statistics as in MesaNet or LLA.

\paragraph{Strong regularization limit.}
As \(\lambda\to\infty\) or \(\|\brho_i\|\to 0\), the probe term is suppressed and the intercept dominates the output.
Parallax and LLA degenerate to Softmax Attention, while Affine MesaNet and Affine Linear Attention degenerates to the Value Averaging mechanism.
Under this limit, the Linear Attention and MesaNet reduce to nothing for the whole term vanishes.
The same parametrization axis explains the relationship between Parallax and LLA, just as MesaNet differs from Linear Attention by probe preconditioning.

\paragraph{Magnitude tension in the affine structure.}\label{para:additive_structure}
Parallax and Affine Linear Attention inherit an additive structure in which the output is a sum of an intercept and a linear evaluation through $\brho_i$.
Since the probe $\brho_i = \bm{W}_R \bx_i$ is parametric rather than an optimal solve, the strength of the linear evaluation relative to the intercept is not guaranteed.
Directionally, only the component of $\brho_i$ aligned with the exact solve $\brho_i^\star$ ($\bm{\Sigma}_i^{-1}\bm{\mu}_i$ for Parallax, $\tilde{\bm H}_i^{-1}\tilde\bq_i$ for Affine Linear) remains functional.
The orthogonal component is unidentifiable and does not contribute toward the correction. Likewise, the norm of the probe no longer respects $\|\brho_i^\star\|$.
In contrast, the magnitude of the intercept term only depends on the weighted averages of $\bv$.

As a result, a poorly aligned or norm-suppressed probe vector renders the covariance correction term functionally inert in prediction, and Parallax collapses in effect toward its Softmax Attention baseline regardless of the affine structure nominally available.
Both the alignment and the norm of the probe depend heavily on optimizer choice, which we analyze empirically in Section~\ref{sec:analysis}.

\subsection{Streaming Algorithm}
\label{sec:streaming}

\begin{figure}
\centering
\begin{subfigure}[c]{0.49\textwidth}
    \centering
    \includegraphics[width=\linewidth]{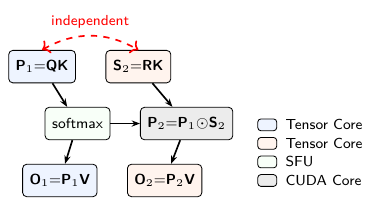}
    \caption{Dependency graph and hardware mapping.}
    \label{fig:dependency}
\end{subfigure}
\hfill
\begin{subfigure}[c]{0.48\textwidth}
    \centering
    \includegraphics[width=\linewidth]{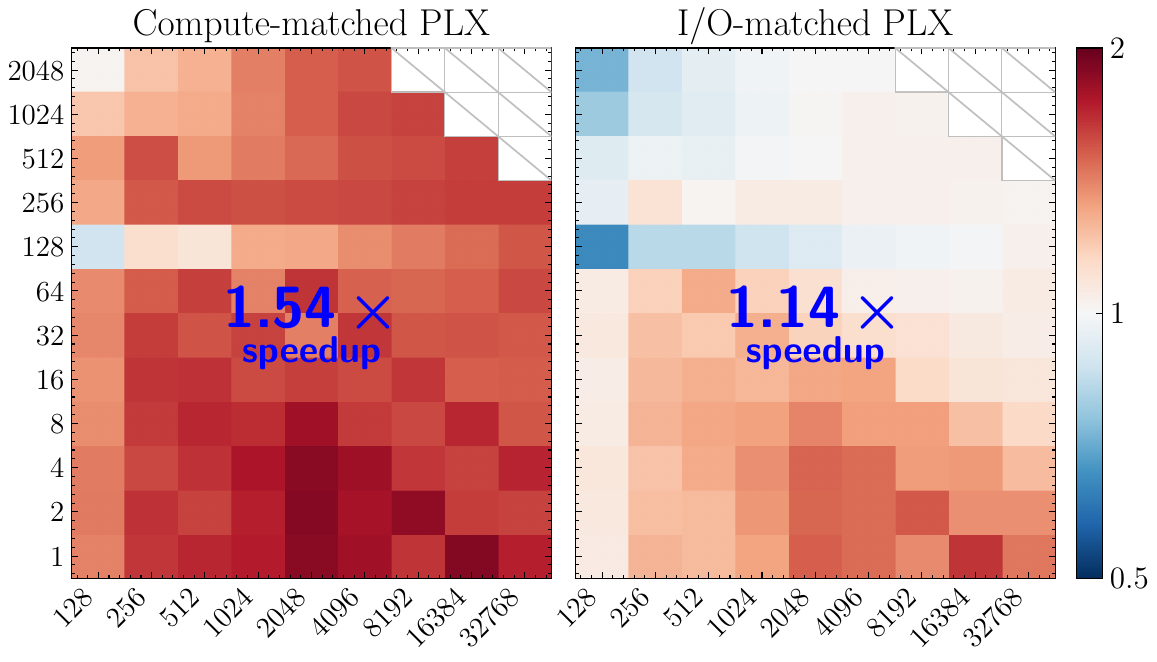}
    \caption{PLX-\texttt{CuTeDSL} vs \texttt{best}(FA2,FA3) in decoding.}
    \label{fig:speedup-heatmap}
\end{subfigure}
\caption{
    Figure~\ref{fig:dependency}: Operator dependency and hardware unit assignment for the Parallax forward.
    Figure~\ref{fig:speedup-heatmap}: Decoding speedup of Parallax kernels in I/O matched and compute-matched setting. X-axis is the context length and Y-axis is the batch$\times$head dimension. The color indicates the latency ratio of the \texttt{best}(FA2,FA3) over Parallax-\texttt{CuTeDSL} kernel, with warmer colors indicating faster decoding with Parallax.
    The upper left tiles with backslash indicates OOM in profiling.
}
\label{fig:streaming-perf}
\end{figure}

Parallax inherits the streaming structure of FlashAttention (FA)~\citep{dao2022flashattentionfastmemoryefficientexact,dao2023flashattention2fasterattentionbetter,shah2024flashattention3} with one additional covariance branch.
In order to stream the computation of equation~\eqref{eq:parallax_forward} in one pass over the KV sequence, we expand the formulation to the following equivalent form,
\begin{align}
    \bo_i^{\textsf{PLX}} = \Bigl(\sum\nolimits_{j\le i}p_{ij}\bv_j\Bigr)\cdot \Bigl(1+\sum\nolimits_{j\le i}p_{ij}\cdot \bk_j^\top\brho_i\Bigr) -\sum\nolimits_{j\le i}\bigl(p_{ij}\cdot \bk_j^\top\brho_i\bigr)\bv_j,
    \label{eq:parallax_expand}
\end{align}where \(p_{ij}\) is the softmax score and \(p_{ij}\cdot \bk_j^\top\brho_i\) is the composite score.
The computation can be implemented with two parallel scoring and accumulation branches.

\begin{wrapfigure}[17]{r}{0.47\textwidth}
\begin{minipage}{\linewidth}
\vspace{-22pt}
\begin{algorithm}[H]
\caption{Parallax forward core computation.}
\label{alg:parallax_forward}
\begin{algorithmic}[1]
\Require Input $\mathbf Q_r, \mathbf R_r, \mathbf K, \mathbf V, L, s$; Output $\mathbf O_r$; Running state $(\mathbf m_r, \mathbf d_{1,2}, \mathbf O_{1,2})$.
\For{$c=1$ to $\lceil L/\mathcal B_c\rceil$}
\State $\mathbf S_1 \gets \mathbf Q_r \mathbf K_c^\top \cdot s${\small\color{gray}\Comment{Apply Masking}}
\State $\mathbf m \gets \max(\mathbf m_r, \rowmax(\mathbf{S}_1))$
\State $\bm \alpha \gets \mathrm{exp2}(\mathbf m_r - \mathbf m)$
\State $\mathbf P_1 \gets \mathrm{exp2}(\mathbf{S}_1 - \mathbf m)$
\State $\mathbf m_r \gets \mathbf m$
\State {\color{red}$\mathbf S_2 \gets \mathbf R_r \mathbf K_c^\top$}{\small\color{gray}\Comment{GEMM in TC}}
\State {\color{red}$\mathbf P_2 \gets \mathbf P_1 \odot \mathbf S_2$}
\State $\mathbf d_1 \gets \bm \alpha \mathbf d_1 + \rowsum(\mathbf P_1)$
\State {\color{red}$\mathbf d_2 \gets \bm \alpha \mathbf d_2 + \rowsum(\mathbf P_2)$}
\State $\mathbf O_1 \gets \bm \alpha \mathbf O_1 + \mathbf P_1 \mathbf V_c$
\State {\color{red}$\mathbf O_2 \gets \bm \alpha \mathbf O_2 + \mathbf P_2 \mathbf V_c$}{\small\color{gray}\Comment{GEMM in TC}}
\EndFor
\State $\mathbf O_r \gets \mathbf O_1 / \mathbf d_1 \cdot {\color{red}(1 + \mathbf d_2 / \mathbf d_1) - \mathbf O_2/\mathbf d_1}$
\end{algorithmic}
\end{algorithm}
\end{minipage}
\end{wrapfigure}
Let \(\mathbf{Q}_r, \mathbf{R}_r\) denote the tiled matrices for a row block of \(\bq\) and \(\brho\) of size \(\mathcal{B}_r\), and \(\mathbf{K}_c, \mathbf{V}_c\) the tiled matrices for a column block of \(\bk\) and \(\bv\) of size \(\mathcal{B}_c\).
The softmax branch maintains the running state \((\mathbf{m}_r, \mathbf{d}_1, \mathbf{O}_1)\) as in FA.
Parallax additionally maintains the state \((\mathbf{d}_2, \mathbf{O}_2)\).
In each loop, the covariance branch uses the same \(\mathbf{K}_c, \mathbf{V}_c\) to compute the unnormalized scores \(\mathbf{S}_2 = \mathbf{R}_r\mathbf{K}_c^\top\), fuses them with the softmax weights as \(\mathbf{P}_2 = \mathbf{P}_1 \odot \mathbf{S}_2\), and then accumulates \(\mathbf{O}_2 = \mathbf{P}_2\mathbf{V}_c\) alongside \(\mathbf{O}_1\).
The final output combines the two running sums according to equation~\eqref{eq:parallax_expand}.

Both branches share the online maximum \(\mathbf{m}_r\), the rescaling factor \(\bm \alpha\) and the \(\mathbf{K}_c, \mathbf{V}_c\) tiles. Therefore Parallax does not require extra I/O in each iteration.
The detailed algorithm is provided in Algorithm~\ref{alg:parallax_forward}, where the additional operations of Parallax are highlighted in red.
The operator dependency graph and hardware mapping are shown in Figure~\ref{fig:dependency}.

\paragraph{Arithmetic intensity.}
The key property of Algorithm~\ref{alg:parallax_forward} is that it increases the arithmetic intensity (\texttt{AI}) over FA, defined as the ratio of floating point operations (FLOPs) to high-bandwidth memory (HBM) traffic in bytes.
Write \(L_q\) and \(L_{kv}\) as the query and KV sequence lengths respectively and \(d_h\) as the head dimension,
\begin{align}
\texttt{AI}^\textsf{FA} \approx \frac{4L_q L_{kv} d_h}{2(L_q + 2n_r L_{kv})d_h} = \frac{2L_q L_{kv}}{L_q + 2n_r L_{kv}},\qquad \texttt{AI}^\textsf{PLX} \approx \frac{8L_q L_{kv} d_h}{2(2L_q + 2n_r L_{kv})d_h} = \frac{2L_q L_{kv}}{L_q + n_r L_{kv}},
\end{align}where \(n_r = \lceil L_q / \mathcal{B}_r \rceil\) is the number of query row blocks.
In the regime where $n_rL_{kv}\gg L_q$, Parallax roughly \emph{doubles the arithmetic intensity} by adding more compute while reusing the same KV stream.
The shift toward a more compute bound operator is what makes decoding a target for kernel level optimization on modern hardware, which we analyze next.

\paragraph{Decode optimization. }\label{para:decode_optimization}
We prototype Parallax decode kernel in \texttt{CuTeDSL}~\citep{sun2025cutedsl} on NVIDIA Hopper GPUs. The design exploits the structural property that the two branches share the same KV stream, so on I/O-bound decode workloads it consumes essentially no additional HBM traffic.
We further exploit that Hopper's tensor core (TC) matmul instructions (WGMMA) operate on tiles of minimum size 64 rows by construction, whereas a decode step supplies only one query row.
The remaining rows of every WGMMA's accumulator would otherwise be idle. Therefore the QK and RK product can be computed jointly, within the same instructions that standard attention already issues.
The same applies to the two PV products in the output accumulation.
The prototype kernel also implements few other optimizations such as persistent split over the KV loop and in-kernel reduction.
We provide a detailed documentation of the kernel optimization in Appendix~\ref{app:decode_optimization}.
\paragraph{Profiling Result. }We profile the prototype kernel against FA2 and FA3 on H200 GPUs at BF16
precision, sweeping batch sizes from $1$ to $2{,}048$ and context lengths
from $128$ to $32{,}768$, both in powers of two. For each configuration we
compare against the best record of FA2 and FA3.
Because Parallax doubles the arithmetic intensity, FA cannot be matched on both FLOPs and HBM traffic simultaneously.
We therefore report two settings: in the \emph{compute-matched} setting, Parallax uses $d_h = 64$ such that it matches the FLOPs of FA at $d_h = 128$;
in the \emph{I/O-matched} setting, both kernels use $d_h = 128$ and Parallax doubles the compute of FA with the same HBM traffic.
Figure~\ref{fig:speedup-heatmap} reports the latency-ratio heatmaps for both settings. The prototype Parallax kernel matches or outperforms FA across all configurations. Additional profiling results are in Appendix~\ref{app:profiling}.
\section{Experiment}
In this section, we empirically validate Parallax on both synthetic and language modeling benchmarks.
We compare Parallax against the Softmax Attention (Attn, Transformer), Mamba, Gated DeltaNet (GDN), MesaNet (Mesa) and Kimi DeltaAttention (KDA)~\citep{kimiteam2025kimilinearexpressiveefficient}.

\subsection{Synthetic Benchmarks}
\label{sec:exp_mad}
\paragraph{{MAD}-Benchmark. }
We evaluate Parallax on the \texttt{MAD}-Benchmark~\citep{pmlr-v235-poli24a}, which consists of six synthetic tasks designed to evaluate the core ability of sequence mixers\footnote{We use the cached data released in the official repository at \url{https://github.com/athms/mad-lab}.}.
Particularly, the In-Context-Recall (\texttt{ICR}), Fuzzy-in-Context-Recall (\texttt{FCR}), Noisy-in-Context-Recall (\texttt{NCR}), and Selective-Copying (\texttt{SC}) tasks assess the model's ability to recall information from the context, while Compression (\texttt{CMP}) and Memorization (\texttt{MEM}) evaluate the model's ability to aggregate and memorize information from the training dataset.

\begin{wraptable}[19]{l}{0.45\textwidth}
\centering
\vspace{0pt}

\begin{subtable}{\linewidth}
\centering
\resizebox{0.85\linewidth}{!}{%
\begin{tabular}{c!{\vrule width 0.6pt}ccccc}
\toprule
Task & Attn & PLX & GDN & Mamba & Mesa \\
\midrule
\texttt{CMP} & 0.342 & 0.332 & 0.325 & \textbf{0.424} & 0.310 \\
\texttt{ICR} & 0.803 & 0.951 & 0.920 & 0.756 & \textbf{0.998} \\
\texttt{FCR} & 0.268 & \textbf{0.356} & 0.110 & 0.065 & 0.219 \\
\texttt{NCR} & 0.861 & 0.937 & 0.907 & 0.713 & \textbf{0.999} \\
\texttt{MEM} & 0.807 & 0.733 & 0.792 & \textbf{0.876} & 0.861 \\
\texttt{SC}  & 0.950 & \textbf{0.988} & 0.939 & 0.950 & 0.568 \\
\midrule
{Avg} & 0.672 & \textbf{0.716} & 0.665 & 0.631 & 0.659 \\
\bottomrule
\end{tabular}
}
\caption{\texttt{MAD} benchmark accuracy.}
\label{tab:results_merged_clla}
\end{subtable}

\vspace{2pt}

\begin{subfigure}{\linewidth}
\centering
\includegraphics[width=0.85\linewidth]{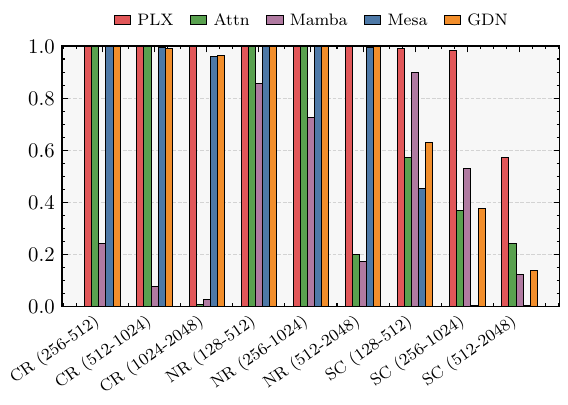}
\vspace{-4pt}
\caption{\texttt{MAD}-challenge recall accuracy.}
\label{fig:mad-challenge}
\end{subfigure}

\end{wraptable}
All models follow a two-layer architecture with sequence mixer and MLP blocks interleaved. Different from prior work, all the models are trained with Muon optimizer.
As reported in Table~\ref{tab:results_merged_clla}, Parallax consistently improves on the recall-oriented tasks (\texttt{ICR}, \texttt{FCR}, \texttt{NCR}, \texttt{SC}) while remaining competitive on the compression and memorization tasks (\texttt{CMP}, \texttt{MEM}), and attains the highest overall accuracy.

To further showcase the advantage of Parallax under more challenging recall conditions, we synthesize an additional set of harder tasks by scaling up the KV pairs and the sequence length on \texttt{ICR}, \texttt{NCR}, and \texttt{SC}, with vocabulary size and context length stressed up to 512 and 2048 respectively.
We apply the same training specification as in the previous experiment.
As shown in Figure~\ref{fig:mad-challenge}, Parallax retains accuracy as the difficulty grows, whereas other baselines degrade dramatically, most visibly on \texttt{SC} at the longest context lengths.
Full experimental setup is reported in Appendix~\ref{app:experiment_setup}.

\subsection{Language Modeling Benchmarks}
\label{sec:exp_lm}

{\setlength{\textfloatsep}{0pt}
\begin{table}[t]
\centering
\footnotesize

\begin{subtable}[t]{0.58\textwidth}
\centering
\vspace{0pt}

\resizebox{\linewidth}{!}{%
\begin{tabular}{ccccccc}
\toprule
\textbf{Scale} & \textbf{Model} & \textbf{Optim.}
& \textbf{Sched.} & \textbf{RoPE $\brho$}
& \textbf{Batch} & \textbf{Tokens} \\
\midrule
\multirow{8}{*}{0.6B}
  & Transformer          & Muon  & WSD    & --     & 3.93M & 78.6B \\
  & Transformer$^\dag$   & Muon  & WSD    & --     & 3.93M & 78.6B \\
  & Parallax$^\dag$    & Muon  & WSD    & \cmark & 3.93M & 78.6B \\
  & Parallax    & Muon  & WSD    & \xmark & 3.93M & 78.6B \\
  & Parallax    & Muon  & WSD    & \cmark & 3.93M & 78.6B \\
  \cmidrule(l){2-7}
  & Transformer          & AdamW & WSD    & --     & 3.93M & 78.6B \\
  & Parallax    & AdamW & WSD    & \cmark & 3.93M & 78.6B \\
  \cmidrule(l){2-7}
  & Transformer          & AdamW & Cosine & --     & 3.93M & 78.6B \\
  & Parallax    & AdamW & Cosine & \cmark & 3.93M & 78.6B \\
\midrule
\multirow{3}{*}{1.7B}
  & Transformer          & Muon  & WSD    & --     & 7.86M & 157.2B \\
  & Parallax    & Muon  & WSD    & \xmark & 7.86M & 157.2B \\
  & Parallax    & Muon  & WSD    & \cmark & 7.86M & 157.2B \\
\bottomrule
\end{tabular}}

\caption{Training configurations.}
\label{tab:train_spec}
\vspace{8pt}
\resizebox{\linewidth}{!}{%
\begin{tabular}{lcc@{\hskip 20pt}lcc}
\toprule
\multicolumn{3}{c}{\textbf{Optimizer}} &
\multicolumn{3}{c}{\textbf{Scheduler}} \\
\cmidrule(lr){1-3}\cmidrule(lr){4-6}
\textbf{Hyperparam.} & \textbf{Muon} & \textbf{AdamW}
&
\textbf{Hyperparam.} & \textbf{WSD} & \textbf{Cosine} \\
\midrule
Learning rate & $5\!\times\!10^{-3}$ & $3\!\times\!10^{-4}$
  & Warmup      & $0\%$  & $1\%$ \\
Weight decay  & 0.1 & 0.1
  & Decay type  & Linear & Cosine \\
Momentum      & 0.95 & 0.9, 0.95
  & Decay start & $80\%$ & $1\%$ \\
Embed/Norm lr  & 0.3$\times$/0.015$\times$ & 1$\times$/1$\times$
  & Final lr    & 0 & 0 \\
\bottomrule
\end{tabular}}

\caption{Optimizer and learning rate scheduler hyperparameters.}
\label{tab:optim}
\end{subtable}%
\hfill
\begin{subfigure}[t]{0.36\textwidth}
\centering
\vspace{-4pt}
\includegraphics[width=\linewidth]{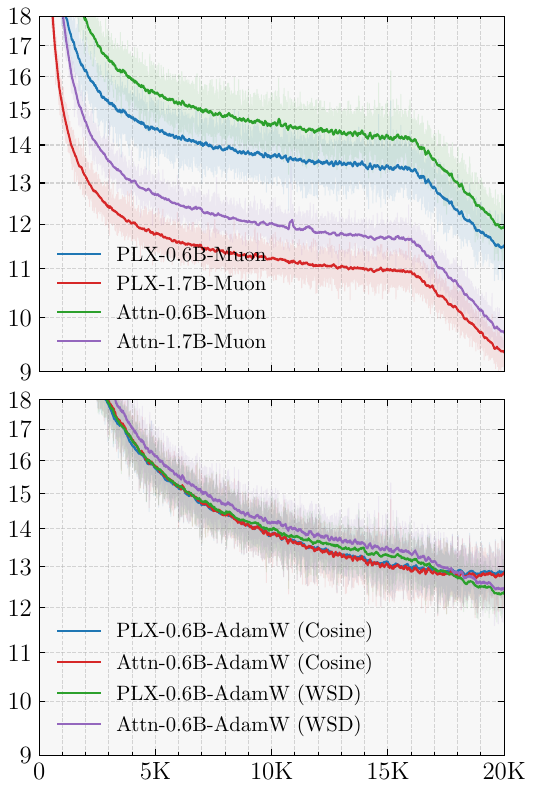}
\caption{Training perplexity.}
\label{fig:loss-curve-combined}

\end{subfigure}

\caption{
    Table~\ref{tab:train_spec} and Table~\ref{tab:optim} detail the training configurations and hyperparameters.
    We apply the same Muon optimizer settings for 0.6B and 1.7B scale models.
    Figure~\ref{fig:loss-curve-combined} shows the training perplexity curves under different optimizers and schedulers.
    Curves are smoothed for visibility.
}
\label{tab:train_configs_and_curves}

\end{table}
}

Having established the recall advantage of Parallax on synthetic tasks, we further evaluate it on LLM pretraining and downstream benchmarks.
\paragraph{Experiment setup. }
We adopt the Qwen-3 architecture~\citep{qwen3} (which applies RMSNorm~\citep{10.5555/3454287.3455397} to \(\bq\) and \(\bk\) vectors) as implemented in the \texttt{torchtitan}~\citep{liang2024torchtitan} repository, and additionally apply RMSNorm to \(\brho\) in each Parallax layer.
All models are pretrained on the Ultra-FineWeb dataset~\citep{wang2025ultrafineweb} with a context length of 4096.
We compare Parallax against the Transformer baseline at both 0.6B and 1.7B parameter scales.
At the 0.6B scale, we also provide the result for KDA, GDN and controlled experiment baselines:
\begin{itemize}
    \item \emph{Parameter-matched Transformer (Transformer$^\dag$).} Parallax introduces extra parameters from the \(\bm{W}_R\) projection. Transformer$^\dag$ adds the same number of parameters to the Transformer baseline by increasing the query head count in GQA.
    This choice maintains the KV size and mirrors how Parallax allocates its parameters due to the similarity between \(\bq\) and \(\brho\) in computation.
    \item \emph{Compute-matched Parallax (Parallax$^\dag$).} Parallax doubles the arithmetic intensity compared to FA. Parallax$^\dag$ halves the head dimension to strictly match the attention layer compute.
    The reduced parameter count is rebalanced by increasing the FFN dimension to match the total parameter count of the standard Parallax.
\end{itemize}
We additionally ablate the effect of applying RoPE~\citep{su2024roformer} to \(\brho\) in Parallax.
The details of training configuration and optimizer hyperparameters are summarized in Table~\ref{tab:train_spec} and Table~\ref{tab:optim}.

\paragraph{Evaluation and benchmarks. }
For perplexity evaluation, we report both LAMBADA (LMB.)~\citep{paperno2016lambada} and WikiText (Wiki)~\citep{merity2017pointer}.
For zero-shot QA and commonsense reasoning evaluation, we report BoolQ~\citep{clark2019boolq}, HellaSwag (HSwag)~\citep{zellers2019hellaswag}, PIQA~\citep{bisk2020piqa}, ARC-easy and ARC-challenge~\citep{clark2018arc}, WinoGrande (Wino.)~\citep{sakaguchi2020winogrande}, OpenBookQA (OBQA)~\citep{mihaylov2018openbookqa}, and SciQ~\citep{welbl2017sciq}.
All evaluations are conducted with the LM Evaluation Harness~\citep{eval-harness}.

\paragraph{Results discussion. }
The results are summarized in Table~\ref{tab:lm-results}.
At the 0.6B scale, Parallax with Muon achieves the best perplexity on both evaluation tasks and the highest average downstream accuracy.
This pattern holds at the 1.7B scale, demonstrating that the gains persist at larger scale under Muon.
We also find that applying RoPE to the \(\brho\) vectors is still beneficial at both scales under Muon, even though the base softmax scores already incorporate positional information.

The two controls answer two distinct questions about the source of the gain:
\begin{itemize}
    \item \emph{Does the gain come from extra parameters?} Transformer$^\dag$ with matched parameters closes only a small fraction of the gap to Parallax, confirming that the improvement is not simply from additional parameters in query-like projections.
    \item \emph{Does the gain require extra Attention compute?} Parallax$^\dag$ with matched attention compute significantly outperforms baseline Transformer and Transformer$^\dag$, ruling out the added compute as a necessary condition.
\end{itemize}
These results provide strong evidence that the gain is driven by the mechanism itself.
Figure~\ref{fig:loss-curve-combined} shows the training perplexity curves under different optimizers and schedulers, all with RoPE applied.
The curves of Muon models show a substantial gap between Parallax and the Transformer, consistent with the downstream evaluation performance.
However, the advantage shrinks markedly or even disappears under AdamW. The performance difference indicates a strong optimizer-architecture interaction, which Section~\ref{sec:analysis} further analyze and discuss.

{\setlength{\textfloatsep}{0pt}
\begin{table}[t]
\small
\centering
\resizebox{\textwidth}{!}{%
\begin{tabular}{cc!{\vrule width 0.6pt}c!{\vrule width 0.6pt}cc!{\vrule width 0.6pt}ccccccccc!{\vrule width 0.6pt}c}
\toprule
\begin{tabular}[c]{@{}c@{}}\textbf{Size}\\ \textbf{Optimizer}\end{tabular}
& \textbf{Model}
& \begin{tabular}[c]{@{}c@{}}\textbf{RoPE}\\ $\brho$\end{tabular}
& \begin{tabular}[c]{@{}c@{}}\textbf{LMB.}\\ ppl$\downarrow$\end{tabular}
& \begin{tabular}[c]{@{}c@{}}\textbf{Wiki}\\ ppl$\downarrow$\end{tabular}
& \begin{tabular}[c]{@{}c@{}}\textbf{LMB.}\\ acc$\uparrow$\end{tabular}
& \begin{tabular}[c]{@{}c@{}}\textbf{BoolQ}\\ acc$\uparrow$\end{tabular}
& \begin{tabular}[c]{@{}c@{}}\textbf{HSwag}\\ acc$\uparrow$\end{tabular}
& \begin{tabular}[c]{@{}c@{}}\textbf{PIQA}\\ acc$\uparrow$\end{tabular}
& \begin{tabular}[c]{@{}c@{}}\textbf{ARC-e}\\ acc$\uparrow$\end{tabular}
& \begin{tabular}[c]{@{}c@{}}\textbf{ARC-c}\\ acc$\uparrow$\end{tabular}
& \begin{tabular}[c]{@{}c@{}}\textbf{Wino.}\\ acc$\uparrow$\end{tabular}
& \begin{tabular}[c]{@{}c@{}}\textbf{OBQA}\\ acc$\uparrow$\end{tabular}
& \begin{tabular}[c]{@{}c@{}}\textbf{SciQ}\\ acc$\uparrow$\end{tabular}
& \begin{tabular}[c]{@{}c@{}}\textbf{Avg}\\ $\uparrow$\end{tabular} \\
\midrule
\multirow{2}{*}{\begin{tabular}[c]{@{}c@{}}Cosine\\AdamW\end{tabular}}
& Transformer    & --- & 31.57 & 26.68 & 34.93 & 61.47 & 46.65 & 70.35 & 60.19 & 30.63 & 52.33 & 34.00 & 80.50 & 52.34 \\
& Parallax  & \cmark & 29.54 & 26.63 & 36.48 & 58.38 & 46.25 & 69.75 & 58.16 & 30.29 & 52.96 & 36.40 & 74.40 & 51.45 \\
\midrule
\multirow{2}{*}{\begin{tabular}[c]{@{}c@{}}WSD\\AdamW\end{tabular}}
& Transformer    & --- & 26.63 & 25.30 & 36.72 & 58.41 & 48.19 & 70.73 & 58.08 & 31.48 & 54.22 & 35.20 & 80.50 & 52.61 \\
& Parallax  & \cmark & 26.59 & 25.01 & 37.40 & 57.16 & 48.63 & 71.44 & 60.90 & 32.00 & 52.17 & 35.60 & 78.80 & 52.68 \\
\midrule
\multirow{6}{*}{\begin{tabular}[c]{@{}c@{}}0.6B\\Muon\end{tabular}}
& Kimi DeltaAttn & --- & 25.16 & 26.81 & 38.29 & 55.29 & 49.30 & 71.11 & 62.12 & 33.19 & 52.57 & 34.20 & 78.50 & 52.73 \\
& Gated DeltaNet & --- & 24.63 & 26.32 & 37.69 & 59.33 & 50.58 & 71.71 & 61.57 & 32.68 & 55.41 & 35.60 & 78.50 & 53.67 \\
& Transformer    & --- & 22.15 & 23.43 & 40.07 & 58.84 & 52.29 & 70.73 & 61.74 & 33.36 & 55.41 & \textbf{37.20} & 81.20 & 54.54 \\
& Transformer$^\dag$    & --- & 22.35 & 23.36 & 39.45 & 56.17 & 52.35 & 71.92 & 63.01 & 35.49 & \textbf{56.59} & 36.80 & 82.30 & 54.90 \\
& Parallax  & \xmark & 19.77 & 22.69 & 41.04 & 60.10 & 53.14 & 71.93 & \textbf{64.27} & 34.73 & 55.64 & 35.20 & \textbf{83.80} & 55.54 \\
& \cellcolor{gray!15}Parallax$^\dag$ & \cellcolor{gray!15}\cmark & \cellcolor{gray!15}20.29 & \cellcolor{gray!15}22.49 & \cellcolor{gray!15}40.21 & \cellcolor{gray!15}\textbf{61.44} & \cellcolor{gray!15}\textbf{54.48} & \cellcolor{gray!15}72.03 & \cellcolor{gray!15}63.80 & \cellcolor{gray!15}\textbf{36.95} & \cellcolor{gray!15}55.01 & \cellcolor{gray!15}35.80 & \cellcolor{gray!15}82.40 & \cellcolor{gray!15}55.79 \\
& \cellcolor{gray!15}Parallax & \cellcolor{gray!15}\cmark & \cellcolor{gray!15}\textbf{18.56} & \cellcolor{gray!15}\textbf{22.25} & \cellcolor{gray!15}\textbf{41.83} & \cellcolor{gray!15}60.24 & \cellcolor{gray!15}53.73 & \cellcolor{gray!15}\textbf{72.52} & \cellcolor{gray!15}63.51 & \cellcolor{gray!15}35.49 & \cellcolor{gray!15}{56.20} & \cellcolor{gray!15}36.80 & \cellcolor{gray!15}83.60 & \cellcolor{gray!15}\textbf{55.99} \\
\midrule
\multirow{3}{*}{\begin{tabular}[c]{@{}c@{}}1.7B\\Muon\end{tabular}}
& Transformer    & --- & 13.07 & 18.11 & 46.77 & \textbf{64.92} & 62.94 & \textbf{76.39} & 69.53 & 42.32 & \textbf{61.01} & \textbf{41.00} & 88.00 & 61.43 \\
& Parallax  & \xmark & 10.85 & 17.27 & 49.54 & 62.72 & 64.34 & 75.79 & 70.33 & \textbf{43.69} & 59.91 & 39.80 & \textbf{89.10} & 61.69 \\
& \cellcolor{gray!15}Parallax & \cellcolor{gray!15}\cmark & \cellcolor{gray!15}\textbf{10.80} & \cellcolor{gray!15}\textbf{17.08} & \cellcolor{gray!15}\textbf{50.26} & \cellcolor{gray!15}64.59 & \cellcolor{gray!15}\textbf{64.54} & \cellcolor{gray!15}\textbf{76.39} & \cellcolor{gray!15}\textbf{73.27} & \cellcolor{gray!15}42.49 & \cellcolor{gray!15}60.77 & \cellcolor{gray!15}\textbf{41.00} & \cellcolor{gray!15}88.70 & \cellcolor{gray!15}\textbf{62.45} \\
\bottomrule
\end{tabular}%
}
\vspace{4pt}
\caption{
    Downstream perplexity and zero-shot accuracy
    ($\uparrow$: higher is better, $\downarrow$: lower is better).
    The average score is computed over the accuracy benchmarks.
    The AdamW groups are at the 0.6B scale.
}
\label{tab:lm-results}
\end{table}
}

\subsection{Mechanism Analysis}
\label{sec:analysis}

\begin{figure}[t]
    \centering
    \begin{subfigure}[t]{0.23\textwidth}
        \centering
        \includegraphics[height=3.2cm,keepaspectratio]{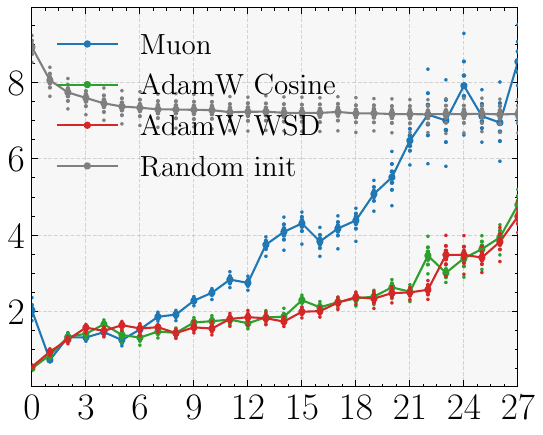}
        \caption{\texttt{COR}}
        \label{fig:cor-ratio}
    \end{subfigure}
    \hfill
    \begin{subfigure}[t]{0.23\textwidth}
        \centering
        \includegraphics[height=3.2cm,keepaspectratio]{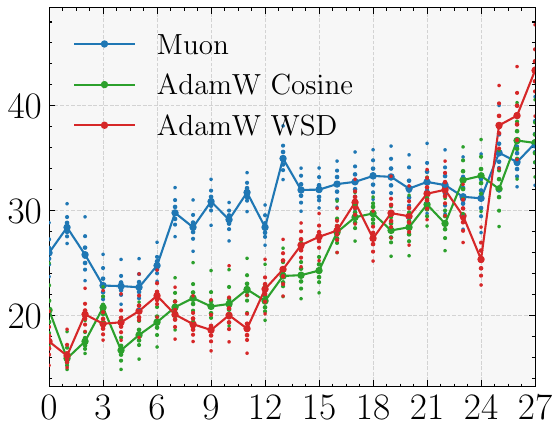}
        \caption{\(\|\texttt{Corr}\|_F\)}
        \label{fig:cor-vkfro}
    \end{subfigure}
    \hfill
    \begin{subfigure}[t]{0.23\textwidth}
        \centering
        \includegraphics[height=3.2cm,keepaspectratio]{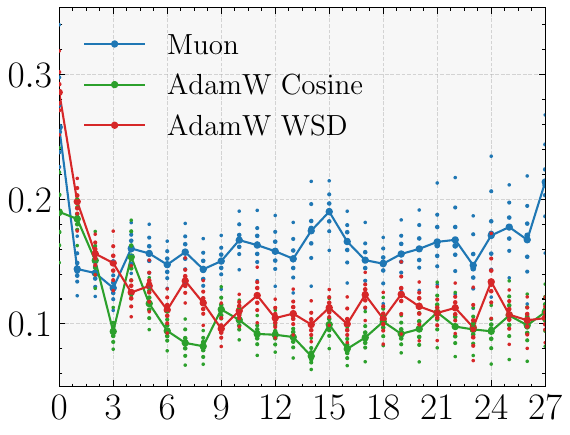}
        \caption{\texttt{CPA}}
        \label{fig:cor-cpa}
    \end{subfigure}
    \hfill
    \begin{subfigure}[t]{0.23\textwidth}
        \centering
        \includegraphics[height=3.2cm,keepaspectratio]{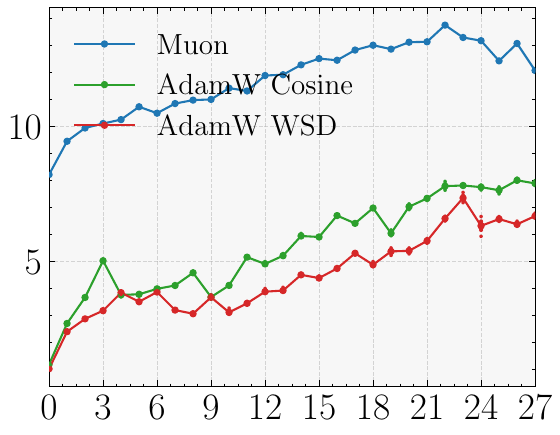}
        \caption{\(\|\brho\|\)}
        \label{fig:cor-rhonorm}
    \end{subfigure}
    \caption{
        From left to right: (\ref{fig:cor-ratio}) correction-to-output ratio \texttt{COR}; (\ref{fig:cor-vkfro}) KV correlation \(\|\texttt{Corr}\|_F\); (\ref{fig:cor-cpa}) covariance-probe alignment \texttt{CPA}; (\ref{fig:cor-rhonorm}) probe norm \(\|\brho\|\). X-axis is the layer index.
        The dots represent the quantile values across heads and positions, and the line represents the mean.
    }
    \label{fig:cor-analysis}
\end{figure}

In this section, we quantitatively analyze the optimizer-architecture interaction of Parallax under different optimizers.
All the analyses are conducted on the 0.6B scale models (RoPE on \(\brho\) applied throughout).
Following the magnitude tension discussion in Section~\ref{para:additive_structure}, the headline quantity is the correction-to-output ratio (\texttt{COR}), defined as
\begin{align}
    \texttt{COR}_i = {\|\bm\Sigma_{KV}^{(i)}\brho_i\|}/{\|\bo_i^{\textsf{SA}}\|},
\end{align}which measures the relative strength of the covariance correction compared to the intercept term in the affine structure.
To decompose any \texttt{COR} gap into directional and magnitude components of the representations, we measure the KV correlation (\texttt{Corr}) and the covariance-probe alignment (\texttt{CPA}) for the directional component, and the probe norm \(\|\brho_i\|\) for the magnitude:
\begin{align}
    \texttt{Corr}_i = \bm\Sigma_{VV}^{(i)-\frac{1}{2}}\bm\Sigma_{KV}^{(i)}\bm\Sigma_{KK}^{(i)-\frac{1}{2}}, \quad
    \texttt{CPA}_i = {\|\bm\Sigma_{KV}^{(i)}\brho_i\|}/{\|\brho_i\|\|\bm\Sigma_{KV}^{(i)}\|_2}.
\end{align}The two metrics are unitless and respectively measure the directional structure of the KV pathway and the alignment of \(\brho_i\) with its leading directions.

The results are shown in Figure~\ref{fig:cor-analysis}.
\texttt{COR} increases with layer depth under all optimizers, with Muon reaching values above \(8\) in the deepest layers while AdamW remains below \(4\).
Compared to random initialization, where \texttt{COR} is uniformly high across all layers, training suppresses the correction in early layers and selectively amplifies it in deeper layers.
This amplification barely recovers the initialization level under AdamW, while the deepest layers exceed it under Muon.
The probe norm shows the largest optimizer gap, which also grows with layer depth.
The directional diagnostics differs as well, with Muon exhibiting higher \(\|\texttt{Corr}\|_F\) and higher \(\texttt{CPA}\) at most layers.
Therefore, the \texttt{COR} gap is not solely a scale effect: Muon also produces richer KV associations and better aligned probes.
The training dynamics of \texttt{COR} is shown in Figure~\ref{fig:corr-ratio-combined}.

\paragraph{Gating behavior.}

The preceding analysis shows that AdamW produces smaller norms and lower \texttt{COR} than Muon.
A natural question is whether this simply reflects a scaling convention, or whether AdamW intrinsically fails to utilize the correction branch in training.
To test this, we use a learnable sigmoid gate \(g_i = \sigma(\bm w_g^\top \bx_i)\) that modulates the probe as \(\brho_i = g_i \cdot \bm W_R \bx_i\).
The gate allows the model to continuously interpolate between suppressing and activating the correction.

Figure~\ref{fig:sigmoid-gate} shows that, under Muon, the model learns to open the gate, and the gated run gradually converges to the same final loss as the non-gated baseline.
Under AdamW, however, the gate value decreases and stabilizes around \(0.26\), and it achieves final performance comparable to Transformer, indicating that the model learns to suppress the covariance correction rather than utilize it.

\paragraph{Spectral structure of weights. }
The activation level differences in both magnitude and direction originate from the weight matrices under different optimizers.
We analyze the stable rank of the projection weight in both Parallax and Softmax Attention.
Beyond the individual projection matrices, we also analyze the bilinear circuits \(\bm W_{QK} = \bm W_Q \bm W_K^\top\), \(\bm W_{OV} = \bm W_O \bm W_V\) and \(\bm W_{RK} = \bm W_R \bm W_K^\top\), which provides a more direct view of the effective transformations.
The bilinear circuits are constructed for each head separately, and the stable rank is averaged across heads and layers.

\begin{wrapfigure}[13]{r}{0.5\textwidth}
\centering
\vspace{-10pt}
\setlength{\tabcolsep}{3pt}
\resizebox{\linewidth}{!}{%
\begin{tabular}{@{}c @{\hskip 4pt} rrr @{\hskip 16pt} rrr@{}}
\toprule
& \multicolumn{3}{c}{\textbf{Softmax Attention}}
& \multicolumn{3}{c}{\textbf{Parallax}} \\
\cmidrule(lr{10pt}){2-4} \cmidrule(lr){5-7}
& Muon & \small AW-Cos & \small AW-WSD
& Muon & \small AW-Cos & \small AW-WSD \\
\midrule
$\bm W_Q$ & 116.0 & \textcolor{darkred}{\textbf{17.4$\downarrow$}} & \textcolor{darkred}{\textbf{21.7$\downarrow$}}
           & 97.4 & \textcolor{darkred}{\textbf{20.9$\downarrow$}} & \textcolor{darkred}{\textbf{20.7$\downarrow$}} \\
$\bm W_K$ & 105.8 & \textcolor{darkred}{\textbf{18.1$\downarrow$}} & \textcolor{darkred}{\textbf{22.9$\downarrow$}}
           & 106.4 & \textcolor{darkred}{\textbf{18.0$\downarrow$}} & \textcolor{darkred}{\textbf{18.3$\downarrow$}} \\
$\bm W_V$ & 145.5 & 148.1 & 138.6
           & \textcolor{darkgreen}{\textbf{199.8$\uparrow$}} & \textcolor{darkgreen}{\textbf{177.2$\uparrow$}} & \textcolor{darkgreen}{\textbf{184.7$\uparrow$}} \\
$\bm W_O$ & 186.6 & 109.8 & 121.6
           & \textcolor{darkgreen}{\textbf{182.0$\uparrow$}} & \textcolor{darkgreen}{\textbf{137.4$\uparrow$}} & \textcolor{darkgreen}{\textbf{144.7$\uparrow$}} \\
$\bm W_R$ & \multicolumn{3}{c}{---}
           & 134.0 & \textcolor{darkred}{\textbf{9.3$\downarrow$}} & \textcolor{darkred}{\textbf{11.1$\downarrow$}} \\
\midrule
$\bm W_{QK}$   & {22.3} & \textcolor{darkred}{\textbf{4.6$\downarrow$}} & \textcolor{darkred}{\textbf{5.9$\downarrow$}}
      & {25.5} & \textcolor{darkred}{\textbf{4.9$\downarrow$}} & \textcolor{darkred}{\textbf{5.1$\downarrow$}} \\
$\bm W_{OV}$   & 24.8 & 21.2 & 21.5
      & \textcolor{darkgreen}{\textbf{34.1$\uparrow$}} & \textcolor{darkgreen}{\textbf{30.4$\uparrow$}} & \textcolor{darkgreen}{\textbf{30.4$\uparrow$}} \\
$\bm W_{RK}$   & \multicolumn{3}{c}{---}
      & {29.1} & \textcolor{darkred}{\textbf{9.4$\downarrow$}} & \textcolor{darkred}{\textbf{9.2$\downarrow$}} \\
\bottomrule
\end{tabular}}
\captionof{table}{Stable ranks of projection matrices and bilinear circuits, averaged across heads and layers.}
\label{tab:srank}
\end{wrapfigure}
The results are summarized in Table~\ref{tab:srank}.
Under Muon, all projection matrices develop substantially higher stable rank than under AdamW, consistent with the prior observations.
Notably, the stable ranks of \(\bm W_Q\), \(\bm W_K\) and \(\bm W_{QK}\) are nearly identical between Parallax and Transformer under every optimizer configuration, confirming the attention scoring structure is shaped by the optimizer alone and unaffected by the architecture.
The rank value of this pattern is marked in red.

For Parallax, \(\bm W_R\) is disproportionately affected. It exhibits the largest optimizer sensitivity of any projection, with a gap exceeding that of \(\bm W_Q\) and \(\bm W_K\).
The \(\bm W_{RK}\) circuit inherits this bottleneck and mirrors the pattern of the QK circuit, where the stable rank is comparable under Muon but collapses under AdamW.
This explains the higher \texttt{CPA} under Muon observed in Figure~\ref{fig:cor-cpa}: a high rank RK circuit allows \(\brho\) to align more effectively with the leading covariance directions.

Beyond the optimizer effect, we also observe a consistent architectural effect where \(\bm W_V\), \(\bm W_O\) and \(\bm W_{OV}\) circuits have higher stable rank for all optimizers under Parallax.
This enrichment reflects the architectural contribution to the value pathway, providing the output projection with a richer set of directions to read from.
The rank value of this pattern is marked in green.

While Muon delivers a substantial advantage over AdamW, we do not claim Muon with WSD is the optimal combination for Parallax.
We provide additional discussion and visualizations in Appendix~\ref{app:training_dynamics} and Appendix~\ref{app:decay}.
\begin{figure}[h]
    \centering
    \begin{subfigure}[t]{0.69\textwidth}
        \centering
        \includegraphics[height=3.5cm,keepaspectratio]{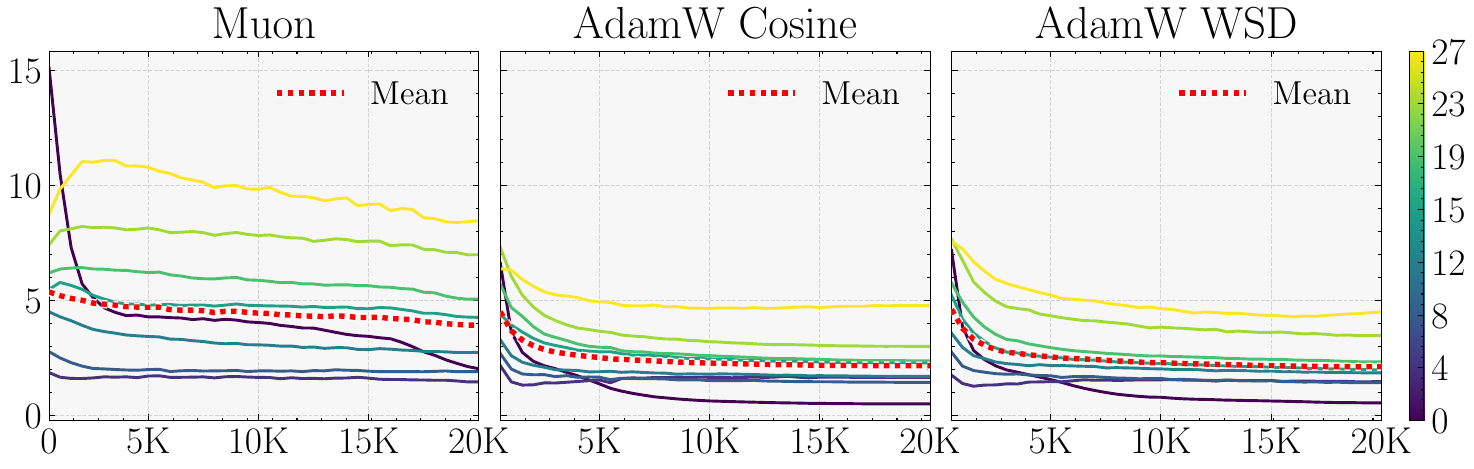}
        \caption{Training dynamics of \texttt{COR} under different optimizers.}
        \label{fig:corr-ratio-combined}
    \end{subfigure}
    \hspace{4pt}
    \begin{subfigure}[t]{0.24\textwidth}
        \centering
        \includegraphics[height=3.2cm,keepaspectratio]{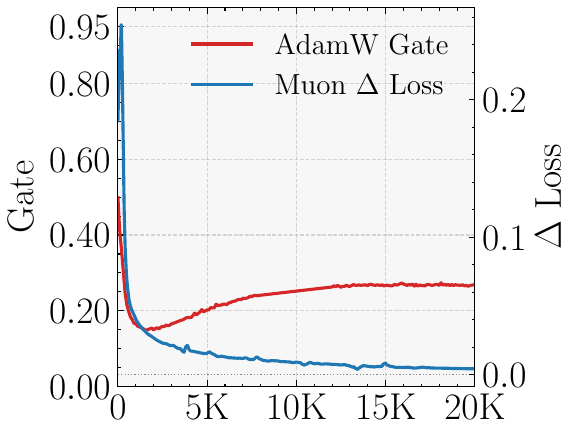}
        \caption{Gating behavior.}
        \label{fig:sigmoid-gate}
    \end{subfigure}
    \caption{
        Training dynamics of \texttt{COR} and the gating behavior.
        Figure~\ref{fig:corr-ratio-combined} shows the training trajectory of \texttt{COR} at sample layers under different optimizers. The colorbar represents the layer index.
        Figure~\ref{fig:sigmoid-gate} shows the gating behavior. The left axis is the gate value and the right axis is the train loss difference.
    }
\end{figure}

\subsection{Parallax Score Distribution Patterns}
Beyond the optimizer driven differences analyzed above, the Parallax mechanism itself produces quantitatively different score distributions from those of Softmax Attention.
Parallax output can be written as a weighted average over values according to Equation~\ref{eq:parallax_forward}. We denote the per-token Parallax weight by \(s_{ij}\in\mathbb{R}\), analogous to the softmax weight \(p_{ij}\),
\begin{align}
s_{ij} = \frac{w_{ij}(1-t_{ij}+\bar t_i)}{\sum_{j'} w_{ij'}(1-t_{ij'}+\bar t_i)} = p_{ij}(1-t_{ij}+\bar t_i).
\end{align}Although \(\sum\nolimits_{j\le i}s_{ij} = 1\) holds, the individual weights \(s_{ij}\) can take negative values and admit values \(|s_{ij}|\gg 1\), which standard softmax weights cannot.
This unbounded range provides additional expressive capacity that the correction branch contributes.
The following analyses are conducted on a held-out validation batch and averaged across queries and heads.
\begin{figure}[t]
    \centering
    \begin{subfigure}[t]{0.32\textwidth}
        \centering
        \includegraphics[height=3.8cm,keepaspectratio]{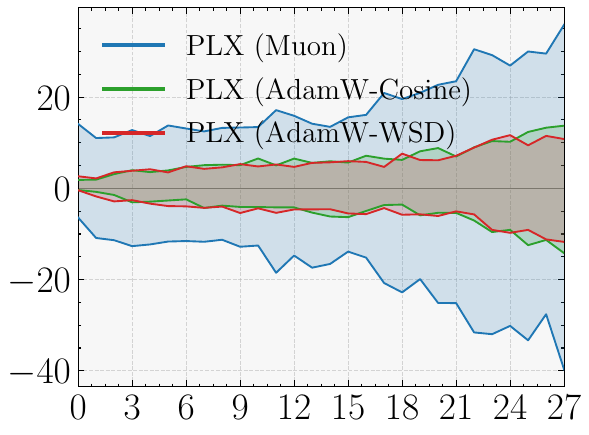}
        \caption{Score ranges.}
        \label{fig:attention-score-range}
    \end{subfigure}
    \begin{subfigure}[t]{0.32\textwidth}
        \centering
        \includegraphics[height=3.8cm,keepaspectratio]{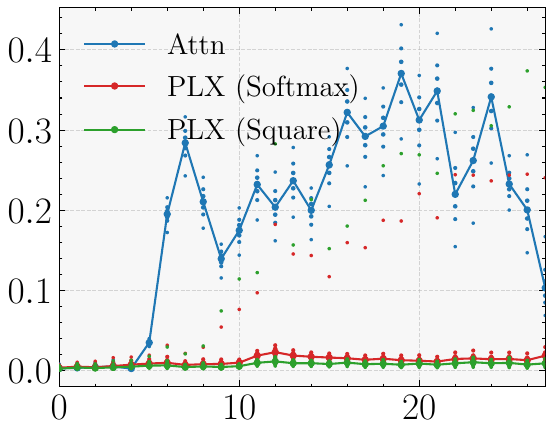}
        \caption{Attention sink.}
        \label{fig:attention-sink}
    \end{subfigure}
    \hspace{-4pt}
    \begin{subfigure}[t]{0.32\textwidth}
        \centering
        \includegraphics[height=3.8cm,keepaspectratio]{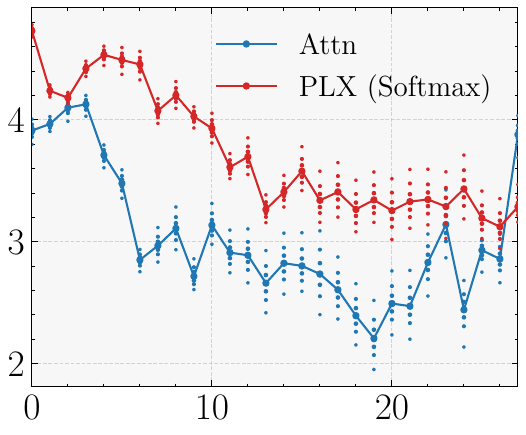}
        \caption{Attention entropy.}
        \label{fig:attention-entropy}
    \end{subfigure}
    \caption{Parallax score patterns.
    Figure~\ref{fig:attention-score-range}, \ref{fig:attention-sink}, and \ref{fig:attention-entropy} respectively show the score range, attention sink and attention entropy patterns of Parallax and Transformer.
    Dots represent the quantile values across heads and positions, and the line represents the mean.
    }
    \label{fig:attention-scores}
\end{figure}
\begin{enumerate}
    \item \textbf{Score range.} We measure the per-query score range for Parallax with three optimizer configurations. Figure~\ref{fig:attention-score-range} report them as a function of layer depth.
    Parallax weights routinely take negative values, allowing the model to actively subtract value components from irrelevant tokens rather than merely de-emphasize them.
    The score range grows with layer depth, with extreme values spanning approximately \(\pm 40\) in the deepest layers under Muon, consistent with the increasing \texttt{COR} and the stronger correction effect.
    \item \textbf{Attention sink.} Softmax Attention is known to concentrate excessive probability mass on the first token \citep{xiao2024efficient}. We quantify the degree of this concentration by measuring the average weight on the first token.
    For Parallax, we measure the sink ratio for both the base softmax \(p_{ij}\) and the combined weights \(s_{ij}\). For the combined weights, we use the squared weight share to handle the negative values:
    \begin{align}
        \texttt{sink}^{\textsf{SA}}_i = \frac{p_{i1}}{\sum\nolimits_{j\le i} p_{ij}} = p_{i1}, \qquad \texttt{sink}^{\textsf{PLX}}_i = \frac{s_{i1}^2}{\sum\nolimits_{j\le i} s_{ij}^2}.
    \end{align}Figure~\ref{fig:attention-sink} shows that Parallax substantially reduces the attention sink in both the base softmax and the combined weights, suggesting that the correction branch may absorb the routing role that Softmax Attention typically discharges onto the first token.
    \item \textbf{Attention entropy.} We measure the dispersion of the softmax distribution in Softmax Attention and the base softmax in Parallax through the Shannon entropy:
    \begin{align}
        H_i = -\sum\nolimits_{j\le i} p_{ij}\log p_{ij}.
    \end{align}Figure~\ref{fig:attention-entropy} shows that Parallax's base softmax entropy is consistently higher than that of the Transformer baseline, showing that Parallax produces more diffuse attention weights.
    Intuitively, Parallax uses the softmax for broader contextual aggregation and offloads fine-grained token discrimination to the correction branch.
\end{enumerate}
We provide additional score distribution visualizations in Appendix~\ref{app:attention_maps}.

\section{Limitations and Future Directions}
\label{sec:limitations}
In this section, we outline several directions opened by this work.

\paragraph{Scaling Parallax.}
Validating the perplexity gain and optimizer-architecture interaction at larger scale, longer context, in combination with components such as MoE and other architectural modifications is left to future work.
The doubled arithmetic intensity opens additional flexibility in tuning the head dimension, head count, and attention to FFN ratio.
Identifying the recipe that best balances model performance and throughput on a given hardware target is an important empirical question.

\paragraph{Optimizing the efficiency of Parallax.}
Parallax inherits the streaming structure of Softmax Attention. Therefore, any contextual sparsity pattern applicable to Softmax Attention, including sliding window, dilated or block sparse, extends directly to Parallax.
It is also structurally compatible with other optimization techniques such as MLA.
We leave the kernel development and performance evaluation of these variants to future work.

\paragraph{Post-training adaptation from pretrained Transformers.}
When initialize $\bm{W}_R = \bm{0}$, Parallax layer behaves identically to a Softmax Attention layer at the start of training.
A pretrained Transformer checkpoint can therefore be converted into a Parallax model by adding the $\bm{W}_R$ weight and fine-tuning.
This contrasts sharply with the Linear Attention family, where no parameter setting recovers Softmax Attention exactly and typically requires retraining to adapt to the new architecture.
Whether the post-training adaptation to Parallax is effective under different optimizer settings is an interesting question we leave open.

\paragraph{Theoretical understanding of the optimizer-architecture interaction.}
In Section~\ref{sec:analysis} we empirically characterize the activation and weight spectra of Parallax under different optimizers and diagnose the performance gap between Muon and AdamW.
However, the precise characterization behind the observed optimizer dependence of Parallax remains an open question.
It also remains to be seen whether this phenomenon happens in other affine mechanisms as discussed in Section~\ref{sec:connections}.

\paragraph{Implications for other attention mechanisms.}
The derivation in Section~\ref{sec:connections} suggests two extensions.
First, Linear Attention, DeltaNet and MesaNet all drop the intercept by construction.
Reintroducing it yields affine variants, two of which appear in Figure~\ref{fig:connections} and an affine DeltaNet is defined analogously.
Whether these affine variants outperform their intercept-free originals, and whether the optimizer-architecture interaction observed here recurs in these variants is a valuable future direction.
Second, the family in Figure~\ref{fig:connections} does not yet include DeltaNet, which should be placed between Linear Attention and MesaNet.
Deriving the nonparametric counterpart of DeltaNet would be a natural extension of the current work.

\section*{Author Contributions}

\emph{Yifei Zuo} conceived the project; developed the mathematics, algorithm, and related derivations; implemented the kernel;
designed and conducted the experiments; and led the writing of the manuscript.
\emph{Dhruv Pai} contributed substantially to the pretraining experiments
and optimizer configuration, and wrote portions of the optimizer-related
discussion.
\emph{Zhichen Zeng} contributed to kernel optimization and authored
portions of the kernel-related sections.
\emph{Alec Dewulf} and \emph{Shuming Hu} contributed valuable discussions to this project. Alec also contributed to the writing of the optimizer-introduction sections.
\emph{Zhaoran Wang} provided advisory guidance throughout the project.

\bibliographystyle{plainnat}
\bibliography{references}
\clearpage
\beginappendix
\section{Theorem}
\label{app:assumptions}

We briefly state the regularity conditions from \citep{zuo2025locallinear} (Appendices A--B).
 
\begin{assumption}[Domain regularity]
\label{asm:domain}
The domain $D \subset \mathbb R^d$ has $C^2$ boundary with principal curvatures uniformly bounded by $\kappa_2$.
\end{assumption}
 
\begin{assumption}[Smoothness]
\label{asm:smooth}
The density $p$ of $X$ satisfies $p \in C^1(D)$ with $p > 0$ on $D$.
The regression function satisfies $f_j \in C^2(D)$ for each output dimension $j$, and the conditional variance $\sigma^2 \in C(D)$.
\end{assumption}
 
\begin{assumption}[Kernel]
\label{asm:kernel}
The kernel $K: \mathbb R^d \to [0,\infty)$ is radial, bounded, and compactly supported on the unit ball $B^d$.
The bandwidth matrix $H = h^2 B$ satisfies $h \to 0$, $nh^d \to \infty$, with condition number $\kappa(B) \le \kappa_1$.
\end{assumption}
 
\begin{assumption}[Boundary gradient]
\label{asm:boundary}
The function $f$ belongs to the class $\mathcal{E}(D, m, M)$: there exists a measurable $\Gamma \subset \partial D$ with positive surface measure such that the inward normal derivative satisfies $|\partial_e f(y)| \ge m$ and the tangential gradient satisfies $\|\nabla_T f(y)\| < M$ for all $y \in \Gamma$, with $m, M$ satisfying a compatibility condition (see \citep{zuo2025locallinear} Definition B.1 and Lemma B.3).
\end{assumption}
 
The $\Omega(1)$ lower bound for the global linear estimator requires only that $f$ is not in the affine class $\mathcal{G} = \{x \mapsto \beta_0 + \beta^\top x\}$ and follows from the projection theorem in $L^2(D)$.
The NW rates follow from pointwise bias--variance analysis with integrated boundary effects (Appendix A of the original).
The LL rates follow from the fact that the local linear estimator achieves $O(\|H\|)$ bias uniformly over $D$, including near $\partial D$, eliminating the boundary-induced $O(\|H\|^{1/2})$ bias of NW. Please refer to the Appendix B of \citep{zuo2025locallinear} for more details.

\section{Additional Derivation of Parallax}
\label{app:derivation_nonnegative_eta}

\subsection{Reformulation of LLA}
\label{app:lla_reformulation}

We derive equation~\eqref{eq:parameterized_lla} from the exact LLA forward in equation~\eqref{eq:lla_forward}.
Starting from
\begin{align}
    \bo_i^{\textsf{LLA}} = \sum_{j\le i}\frac{w_{ij}(1 - t_{ij})}{\omega_i - \bm\mu_i^\top \brho_i}\bv_j,
    \qquad t_{ij} = \brho_i^\top \bz_{ij},
\end{align}
divide the numerator and denominator by $\omega_i = \sum_{j\le i} w_{ij}$, and write $p_{ij} = w_{ij}/\omega_i$ for the softmax weight.
Using $\bm\mu_i = \sum_{j\le i} w_{ij} \bz_{ij}$, we have $\bm\mu_i^\top \brho_i / \omega_i = \E_{\bm p_i}[t_{ij}] = \bar t_i$, so
\begin{align}
    \bo_i^{\textsf{LLA}} = \frac{\E_{\bm p_i}\bigl[(1-t_{ij})\bv_j\bigr]}{1 - \bar t_i}.
    \label{eq:lla_softmax_form}
\end{align}

We expand the numerator. Since $\bz_{ij} = \bk_j - \bq_i$,
\begin{align}
    \E_{\bm p_i}[t_{ij} \bv_j]
    = \E_{\bm p_i}[\bv_j \bz_{ij}^\top]\,\brho_i
    = \bigl(\E_{\bm p_i}[\bv_j \bk_j^\top] - \bar\bv_i \bq_i^\top\bigr)\brho_i.
\end{align}
The cross moment factors as $\E_{\bm p_i}[\bv_j \bk_j^\top] = \bm\Sigma_{KV}^{(i)} + \bar\bv_i \bar\bk_i^\top$, which combined with $\bar\bz_i = \bar\bk_i - \bq_i = \E_{\bm p_i}[\bz_{ij}]$ gives
\begin{align}
    \E_{\bm p_i}[\bv_j \bz_{ij}^\top] = \bm\Sigma_{KV}^{(i)} + \bar\bv_i \bar\bz_i^\top.
\end{align}
Therefore
\begin{align}
    \E_{\bm p_i}[t_{ij} \bv_j]
    = \bm\Sigma_{KV}^{(i)}\brho_i + \bar\bv_i\,(\bar\bz_i^\top \brho_i)
    = \bm\Sigma_{KV}^{(i)}\brho_i + \bar t_i\,\bar\bv_i,
\end{align}
and combining with $\E_{\bm p_i}[\bv_j] = \bar\bv_i = \bo_i^{\textsf{SA}}$ in equation~\eqref{eq:lla_softmax_form},
\begin{align}
    \E_{\bm p_i}\bigl[(1 - t_{ij})\bv_j\bigr]
    = (1 - \bar t_i)\,\bo_i^{\textsf{SA}} - \bm\Sigma_{KV}^{(i)}\brho_i.
\end{align}
Dividing by $1 - \bar t_i$ and using $1/(1 - \bar t_i) = 1 + \eta_i$ recovers the reformulation
\begin{align}
    \bo_i^{\textsf{LLA}}
    = \bo_i^{\textsf{SA}} - (1 + \eta_i)\,\bm\Sigma_{KV}^{(i)}\brho_i.
\end{align}

\subsection{Proof of Proposition~\ref{prop:nonnegative_eta}}
\label{app:proof_nonnegative_eta}

\begin{proof}
We start by decomposing $\bm\Sigma_i$ via the variance identity
\begin{align}
    \sum_{j\le i} w_{ij}\,\bz_{ij}\bz_{ij}^\top
    = \omega_i\,\E_{\bm p_i}[\bz_{ij}\bz_{ij}^\top]
    = \omega_i\,\mathrm{Var}_{\bm p_i}(\bz_{ij}) + \omega_i\,\bar\bz_i \bar\bz_i^\top,
\end{align}
which yields the rank one decomposition
\begin{align}
    \bm\Sigma_i = \bm A_i + \omega_i\,\bar\bz_i \bar\bz_i^\top,
    \qquad
    \bm A_i = \omega_i\,\mathrm{Var}_{\bm p_i}(\bz_{ij}) + \lambda \bm I \succ 0,
    \label{eq:sigma_decomp}
\end{align}
where positive definiteness of $\bm A_i$ follows from $\lambda > 0$.
By the Sherman--Morrison formula,
\begin{align}
    \bm\Sigma_i^{-1}
    = \bm A_i^{-1}
      - \frac{\omega_i\,\bm A_i^{-1} \bar\bz_i \bar\bz_i^\top \bm A_i^{-1}}{1 + \omega_i\,\bar\bz_i^\top \bm A_i^{-1} \bar\bz_i}.
    \label{eq:sherman_morrison}
\end{align}
Define
\begin{align}
    u_i \;:=\; \omega_i\,\bar\bz_i^\top \bm A_i^{-1} \bar\bz_i.
\end{align}
Since $\bm A_i^{-1} \succ 0$, the quadratic form satisfies $u_i \ge 0$, with equality if and only if $\bar\bz_i = \bm 0$.

Using $\bm\mu_i = \omega_i\,\bar\bz_i$ and applying equation~\eqref{eq:sherman_morrison},
\begin{align}
    \brho_i
    = \bm\Sigma_i^{-1}\bm\mu_i
    = \omega_i\,\bm\Sigma_i^{-1}\bar\bz_i
    = \omega_i\,\bm A_i^{-1}\bar\bz_i\,\Bigl(1 - \frac{u_i}{1 + u_i}\Bigr)
    = \frac{\omega_i\,\bm A_i^{-1}\bar\bz_i}{1 + u_i}.
\end{align}
Substituting back into $\bar t_i = \bar\bz_i^\top \brho_i$ gives
\begin{align}
    \bar t_i
    = \frac{\omega_i\,\bar\bz_i^\top \bm A_i^{-1}\bar\bz_i}{1 + u_i}
    = \frac{u_i}{1 + u_i} \in [0, 1),
\end{align}
which immediately implies
\begin{align}
    \eta_i
    = \frac{\bar t_i}{1 - \bar t_i}
    = u_i
    = \omega_i\,\bar\bz_i^\top \bm A_i^{-1} \bar\bz_i \;\ge\; 0.
\end{align}
The final expression equals $\omega_i$ times the squared Mahalanobis distance from $\bq_i$ to the conditional key mean $\bar\bk_i$ under the metric $\bm A_i^{-1}$, justifying the geometric interpretation stated in the main text.
\end{proof}

\section{Parallax Decode Kernel}

\begin{figure}[h]
    \centering
    \begin{subfigure}{\linewidth}
        \centering
        \includegraphics[width=0.7\textwidth]{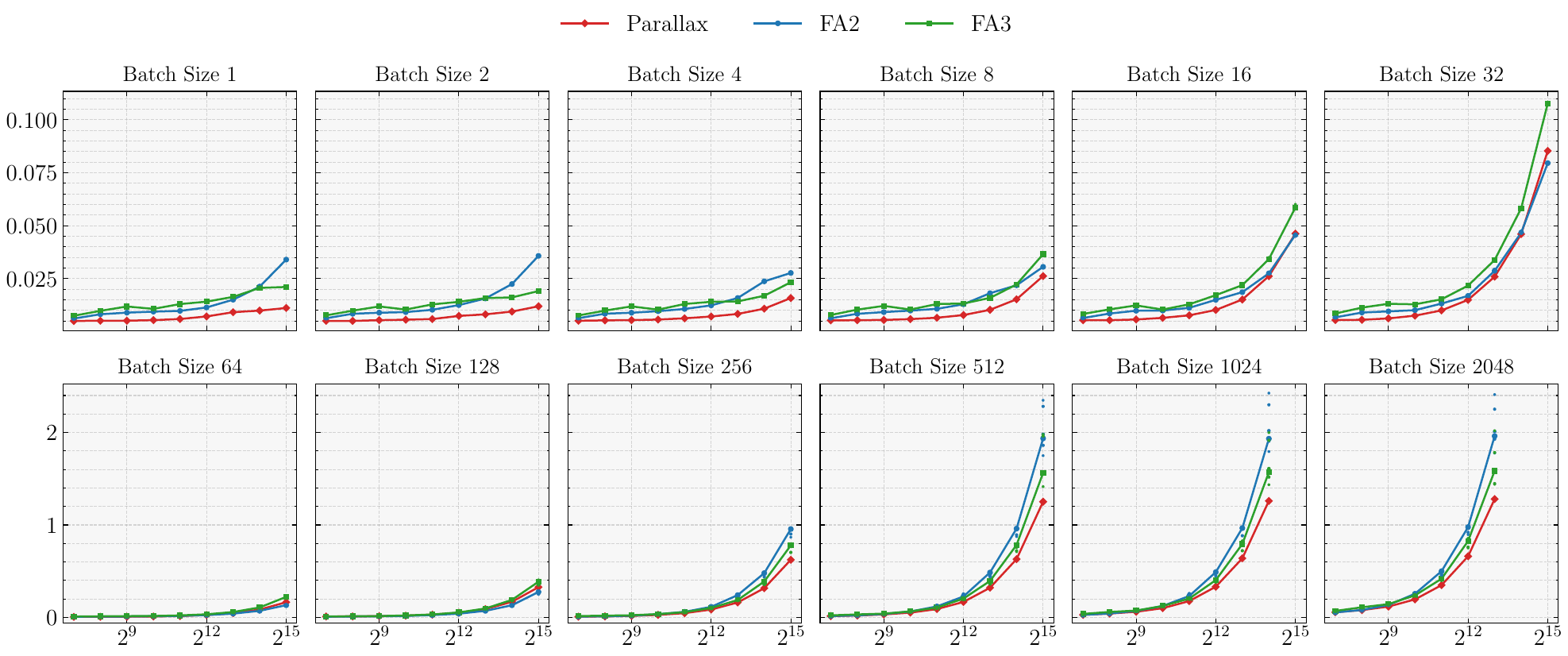}
        \caption{Latency in cuda graph ($d_h=64$).}
        \label{fig:latency-grid-d64-cudagraph}
    \end{subfigure}
    \begin{subfigure}{\linewidth}
        \centering
        \includegraphics[width=0.7\textwidth]{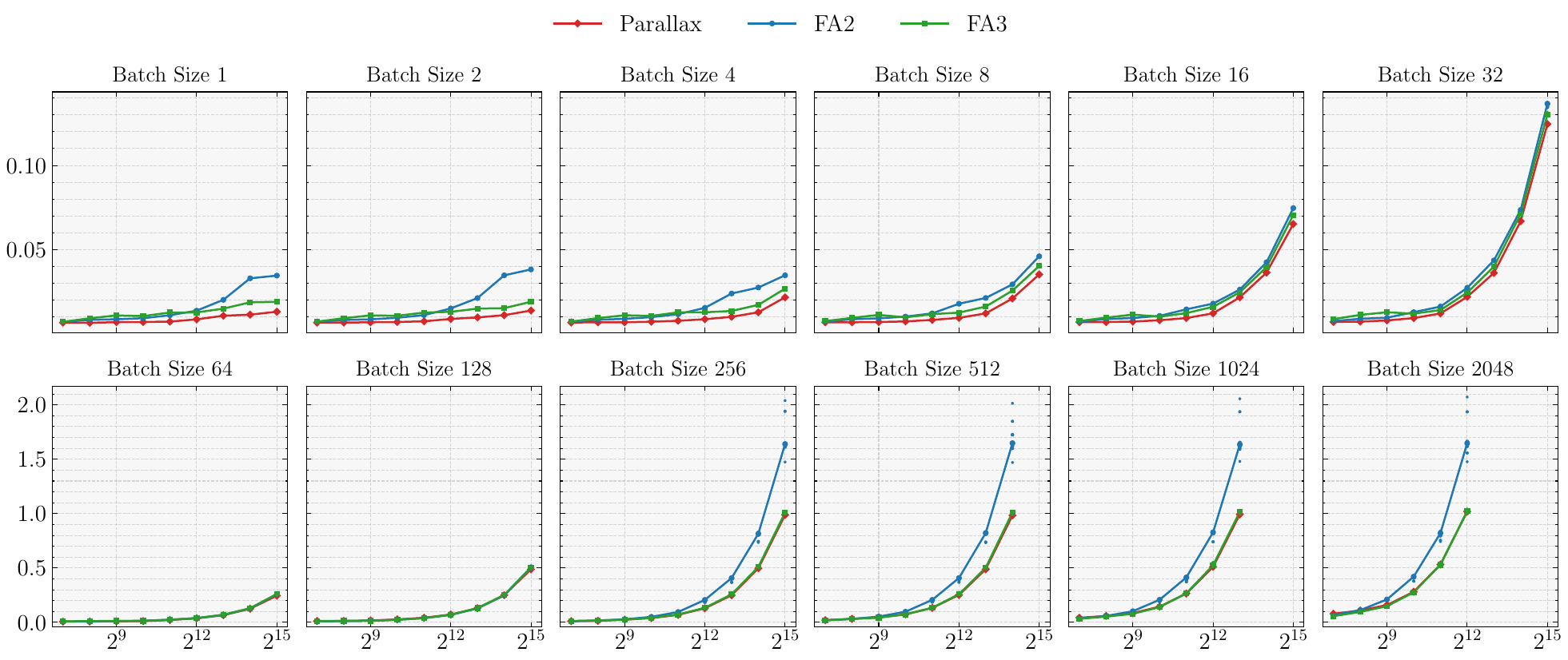}
        \caption{Latency in cuda graph ($d_h=128$).}
        \label{fig:latency-grid-d128-cudagraph}
    \end{subfigure}
    \begin{subfigure}{\linewidth}
        \centering
        \includegraphics[width=0.7\textwidth]{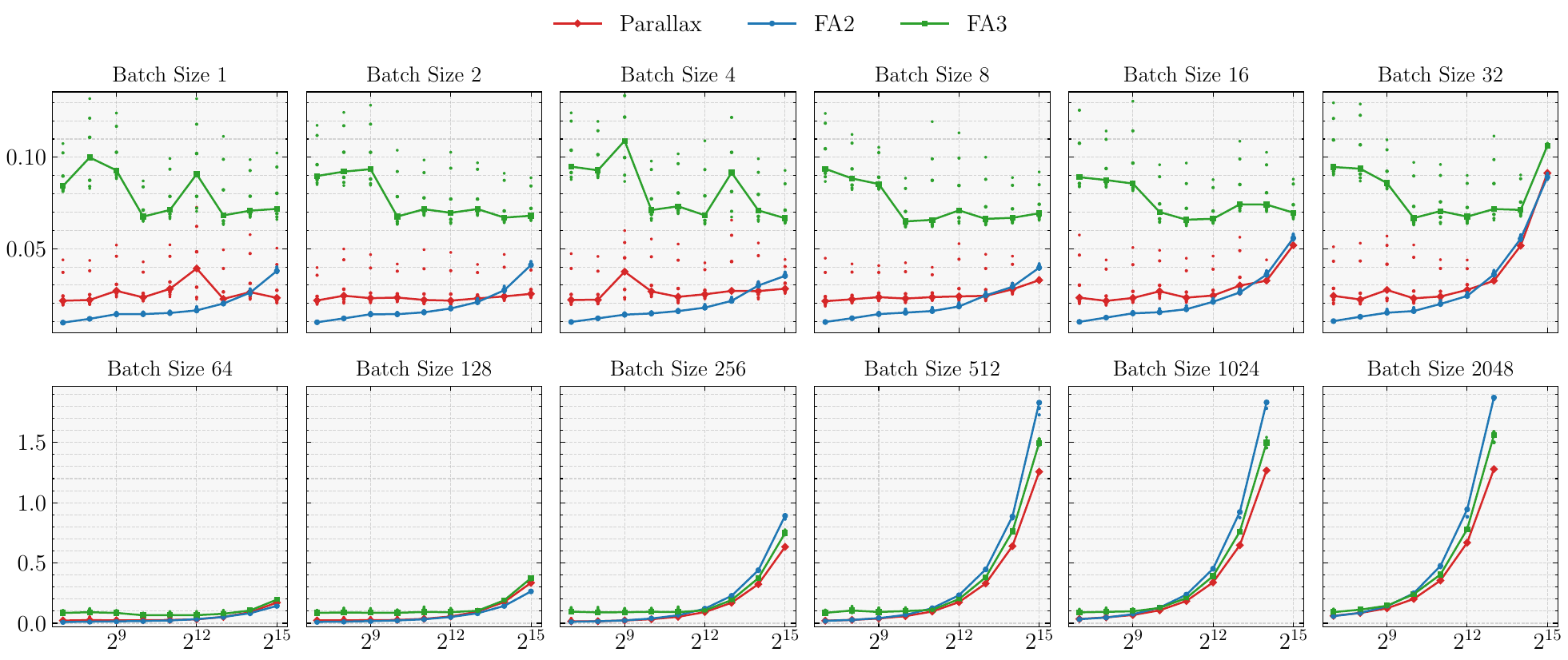}
        \caption{Latency in \texttt{do\_bench} ($d_h=64$).}
        \label{fig:latency-grid-d64-dobench}
    \end{subfigure}
    \begin{subfigure}{\linewidth}
        \centering
        \includegraphics[width=0.7\textwidth]{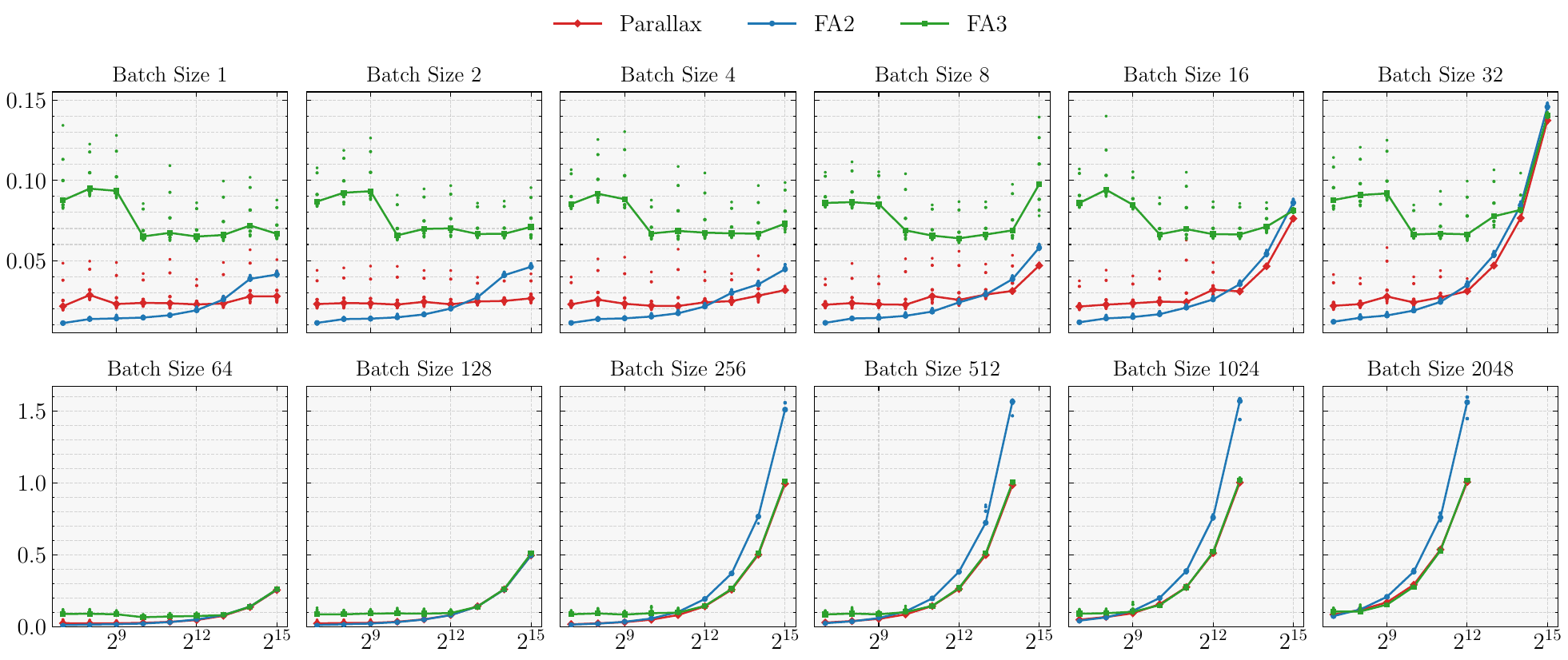}
        \caption{Latency in \texttt{do\_bench} ($d_h=128$).}
        \label{fig:latency-grid-d128-dobench}
    \end{subfigure}
    \caption{Kernel latency comparison. X-axis is the context length and Y-axis is the latency in milliseconds. The latency is measured using a cuda graph (top) and using \texttt{do\_bench} (bottom). The cuda graph measurement isolates the kernel latency, while the \texttt{do\_bench} captures the end-to-end latency including kernel launch overhead.}
    \label{fig:latency-combined}
\end{figure}

\subsection{Kernel Optimization Details}
\label{app:decode_optimization}
This section extends the discussion in Section~\ref{para:decode_optimization} by describing the three major optimizations that optimize the decode latency of Parallax prototype kernel in \texttt{CuTeDSL} on H200:
\begin{enumerate}
    \item \textbf{WGMMA sharing.} Each compute thread array (CTA) loads $\mathbf Q_r$ and $\mathbf R_r$ into a single shared memory tile, with $\mathbf Q_r$ as the first row and $\mathbf R_r$ as the second. The WGMMA then emits $\mathbf S_1$ and $\mathbf S_2$ in Algorithm~\ref{alg:parallax_forward} jointly in the same accumulator. After producing $\mathbf P_1$, we build $\mathbf P_2 = \mathbf P_1 \odot \mathbf S_2$ in registers and stack it with $\mathbf P_1$ in shared memory. The PV WGMMA then emits $\mathbf O_1$ and $\mathbf O_2$ jointly.
    The covariance branch therefore costs one extra row of register accumulators per CTA, with no additional HBM traffic.
    \item \textbf{Persistent split over the KV loop.} Decoding presents only $BH$ query rows, often well below the 132 SMs of an H200 in practical configurations. We launch a persistent grid of $(B, H, S)$ CTAs, where the $S$ number of CTAs share a $(B, H)$ partition the $\lceil L / \mathcal{B}_c \rceil$ tile loop of Algorithm~\ref{alg:parallax_forward}. The split count $S$ is set so that the launch fits one wave on the device and is rounded to a power of two so that the cross split reduction can be vectorized.
    \item \textbf{In-kernel reduction.} Each CTA writes its unnormalized partials $(\mathbf m, \mathbf d_1, \mathbf d_2, \mathbf O_1, \mathbf O_2)$ to a small \texttt{fp32} HBM workspace and atomically increments a per $(B, H)$ counter. The CTA that observes the final increment is elected the merger: it reads the $S$ partials, runs the log-sum-exp rescaling in \texttt{fp32}, evaluates $(1 + \mathbf d_2/\mathbf d_1)\mathbf O_1/\mathbf d_1 - \mathbf O_2/\mathbf d_1$, and writes the output row in the same kernel launch.
    A compile time branch on $S = 1$ skips the workspace round trip and writes the output directly from registers, matching the latency of a single CTA kernel on short context shapes.
\end{enumerate}

\subsection{Additional Profiling Results}
\label{app:profiling}
We provide the raw profiling results in Figure~\ref{fig:speedup-heatmap} on H200 in Figure~\ref{fig:latency-combined}.
The cuda graph measurement isolates the kernel latency, while the \texttt{do\_bench} measurement from \texttt{Triton} captures the end-to-end latency including kernel launch overhead, which FA3 suffers more from.
The heatmaps in Figure~\ref{fig:speedup-heatmap} are derived from the cuda graph measurements, which show a more consistent speedup pattern across different shapes.

\section{Parallax Backward}
\label{app:bwd}

We derive the closed form gradients $\dd\bm{Q}, \dd\bm{R}, \dd\bm{K}, \dd\bm{V}$ of the Parallax forward in equation~\eqref{eq:parallax_expand}, then briefly describe how the gradients can be streamed in the same row tile and column tile structure as the FA backward.

Following Section~\ref{sec:prelim}, let $w_{ij} = \exp(\bq_i^\top \bk_j / h)$, $\omega_i = \sum_{j\le i} w_{ij}$ and $p_{ij} = w_{ij}/\omega_i$. Recall the composite score $t_{ij} = \brho_i^\top \bz_{ij}$ and its softmax weighted mean $\bar t_i = \E_{\bm p_i}[t_{ij}]$. The Parallax forward in equation~\eqref{eq:parallax_expand} admits the equivalent reweighted softmax form
\begin{align}
    \bo_i = \sum_{j\le i} p_{ij}\,(1 + \bar t_i - t_{ij})\,\bv_j,
    \label{eq:bwd_reweighted}
\end{align}
which exposes the two channels through which $\bq_i$ and $\brho_i$ enter the output. The query shapes the softmax weight $p_{ij}$, and the probe modulates the per token coefficient $1 + \bar t_i - t_{ij}$.
Let $\dd\bm{O}_i$ denote the upstream gradient at row $i$. Define three projections of $\dd\bm{O}_i$ and the centered gradient $\delta_{ij}$:
\begin{align}
    \tau_i = \dd\bm{O}_i^\top \bo_i,
    \quad
    \beta_i = \dd\bm{O}_i^\top \bar\bv_i,
    \quad
    a_{ij} = \dd\bm{O}_i^\top \bv_j,
    \quad
    \delta_{ij} = a_{ij} - \beta_i,
    \label{eq:bwd_projections}
\end{align}
where $\bar\bv_i = \mathbb{E}_{\bm p_i}[\bv_j]$ and $\tau_i, \beta_i$ are row scalars that compress the dependence of the loss on the output and the value mean, while $a_{ij}, \delta_{ij}$ resolve the per token contribution.
Define two sets of coefficients for the query and probe channels respectively:
\begin{align}
    g^{(1)}_{ij} = p_{ij}\bigl[a_{ij} - \tau_i + (\bar t_i - t_{ij})\,\delta_{ij}\bigr],\qquad g^{(2)}_{ij} = -p_{ij}\,\delta_{ij}.
\end{align}
Differentiating equation~\eqref{eq:bwd_reweighted} and applying the standard softmax derivative gives
\begin{align}
    \dd\bm{Q}_i &= h^{-1}\sum_{j\le i}g^{(1)}_{ij}\,\bk_j,
    \label{eq:bwd_dq}\\
    \dd\bm{R}_i &= \sum_{j\le i}g^{(2)}_{ij}\,\bk_j,
    \label{eq:bwd_dr}\\
    \dd\bm{K}_j &= h^{-1}\sum_{i\ge j}g^{(1)}_{ij}\,\bq_i + \sum_{i\ge j}g^{(2)}_{ij}\,\brho_i,
    \label{eq:bwd_dk}\\
    \dd\bm{V}_j &= \sum_{i\ge j} p_{ij}\,(1 + \bar t_i - t_{ij})\,\dd\bm{O}_i.
    \label{eq:bwd_dv}
\end{align}
These gradients admit the same row tile and column tile streaming structure as the FA backward. The forward writes out the cache $(\bo_i, \bar\bv_i, \bar t_i, \omega_i, m_i)$ per row, adding only $d + 1$ values over the FA cache.
Due to the difference in reduction direction, the backward kernel is split into two passes:
\begin{itemize}
    \item \emph{Row tile pass. }Loads $(\mathbf{Q}_r, \mathbf{R}_r, \dd\mathbf{O}_r)$ and the cached state, streams over column blocks of $(\mathbf{K}, \mathbf{V})$, and accumulates $\dd\mathbf{Q}_r, \dd\mathbf{R}_r$ in parallel.
    \item \emph{Column tile pass. }Loads $(\mathbf{K}_c, \mathbf{V}_c)$ once and streams over row blocks of $(\mathbf{Q}, \mathbf{R}, \dd\mathbf{O})$ in reverse order to accumulate $\dd\mathbf{K}_c, \dd\mathbf{V}_c$ in parallel.
\end{itemize}
The Parallax backward therefore can be implemented in an I/O-aware streaming algorithm like the FA backward.

\section{Synthetic Experiment Setup}
\label{app:mad-setup}

We follow the pipeline of \citet{pmlr-v235-poli24a} without data modification, swapping only the sequence mixer block.
Every model is a stack of two mixer and SwiGLU MLP blocks with hidden size $d=128$ in \texttt{bf16} precision.
Training proceeds for 60 epochs of Muon with the WSD schedule ($0\%$ warmup, last $20\%$ linearly decayed), matching the optimizer settings in Table~\ref{tab:optim}.
We sweep the peak learning rate over $\{5\times 10^{-3},\,1\times 10^{-3},\,5\times 10^{-4}\}$ and report the best checkpoint per task.
All other settings, including data splits, vocabulary, sequence length, number of KV pairs, and evaluation metric, follow the official release at \url{https://github.com/athms/mad-lab}.
For the harder sweep in Figure~\ref{fig:mad-challenge}, we vary only the data generation parameters: the vocabulary size is raised up to $512$ and the context length up to $2048$ on \texttt{ICR}, \texttt{NCR}, and \texttt{SC}, while all other training hyperparameters remain unchanged.

\section{Pretraining Experiment Setup}
\label{app:experiment_setup}

This appendix complements Section~\ref{sec:exp_lm} and Tables~\ref{tab:train_spec} and~\ref{tab:optim} by reporting the implementation details that are not surfaced inline.
Each training run is conducted on a node consisting of \(8\times\)H100 GPUs.

\subsection{Backbone Architecture}
\label{app:arch}

The Qwen-3 decoder backbone is shared across all language modeling runs.
The full set of hyperparameters of 0.6B model is listed in Table~\ref{tab:arch}. The 1.7B model only doubles the hidden dimension and the SwiGLU MLP hidden size.
Every run ties the input and output embeddings, applies RMSNorm to $\bq$ and $\bk$ vectors, and uses RoPE with base $\theta=10^6$.
For Parallax runs, the additional projection $\bm W_R$ shares the head dimension and head grouping with $\bm W_Q$, and an RMSNorm is applied to $\brho$.
When RoPE is applied to $\brho$, it uses the same base $\theta$ as the queries and keys.

The Transformer$^\dag$ variant raises the query head count while keeping the KV head count fixed, so that its parameter count matches that of Parallax under GQA.
We also train a parameter-matched Transformer that extends the FFN hidden dimension to $3712$ instead of the head count, and find that it performs similarly to Transformer$^\dag$.
For Parallax$^\dag$, the head dimension is halved while the head count remains unchanged.
The reduced parameter count is compensated by extending the FFN hidden dimension to $4480$, so that the total parameter count also matches that of Parallax.

\begin{table}[h]
\centering
\small
\setlength{\tabcolsep}{6pt}
\renewcommand{\arraystretch}{0.95}
\resizebox{\linewidth}{!}{%
\begin{tabular}{lcccccccccc}
\toprule
\textbf{Config} & \textbf{Hidden} & \textbf{FFN} & \textbf{Head dim} & \textbf{Q head} & \textbf{KV head} & \textbf{Layers} & \textbf{Vocab} & \textbf{RoPE} & \textbf{Embed} & \textbf{QKNorm} \\
\midrule
Transformer        & 1024 & 3072          & 128         & 16           & 8 & 28 & 152k & $10^6$ & Tied & RMSNorm \\
Parallax           & 1024 & 3072          & 128         & 16           & 8 & 28 & 152k & $10^6$ & Tied & RMSNorm \\
Transformer$^\dag$ & 1024 & 3072          & 128         & \textbf{24}  & 8 & 28 & 152k & $10^6$ & Tied & RMSNorm \\
Parallax$^\dag$    & 1024 & \textbf{4480} & \textbf{64} & 16           & 8 & 28 & 152k & $10^6$ & Tied & RMSNorm \\
\bottomrule
\end{tabular}}
\vspace{2pt}
\caption{Backbone hyperparameters of the 0.6B model. Differences from the Transformer baseline are bolded.}
\label{tab:arch}
\end{table}

\subsection{Optimizer and Scheduler}
\label{app:optim-detail}

Both optimizers apply gradient norm clipping at $1.0$.
The Muon implementation uses five Newton--Schulz iterations with the standard quintic coefficient and spectral scaling.
Parameters that are not orthogonalizable, including embeddings, norms, and biases, fall back to an Adam style scalar update with $(\beta_1,\beta_2)=(0.8,0.95)$ and $\varepsilon=10^{-7}$.
The AdamW implementation uses $\varepsilon=10^{-8}$ with decoupled weight decay~\citep{loshchilov2019decoupled}.
All language modeling runs train for $20{,}000$ optimizer steps.
The 1.7B runs double the global batch size relative to the 0.6B runs (Table~\ref{tab:train_spec}), so that the total token count scales to approximately $157.2$\,B at the same step count.

\subsection{Precision and Parallelism}
\label{app:sys}

All runs use fully sharded data parallelism without tensor, context, or pipeline parallelism in H100.
We apply \texttt{torchao} dynamic \texttt{fp8} to all linear layers except the LM head, which remains in \texttt{bf16} for numerical stability.

\section{Additional Experiment Results}
\label{app:additional_results}

\begin{figure}[h]
    \centering
    \begin{subfigure}{\linewidth}
        \centering
        \includegraphics[width=0.75\textwidth]{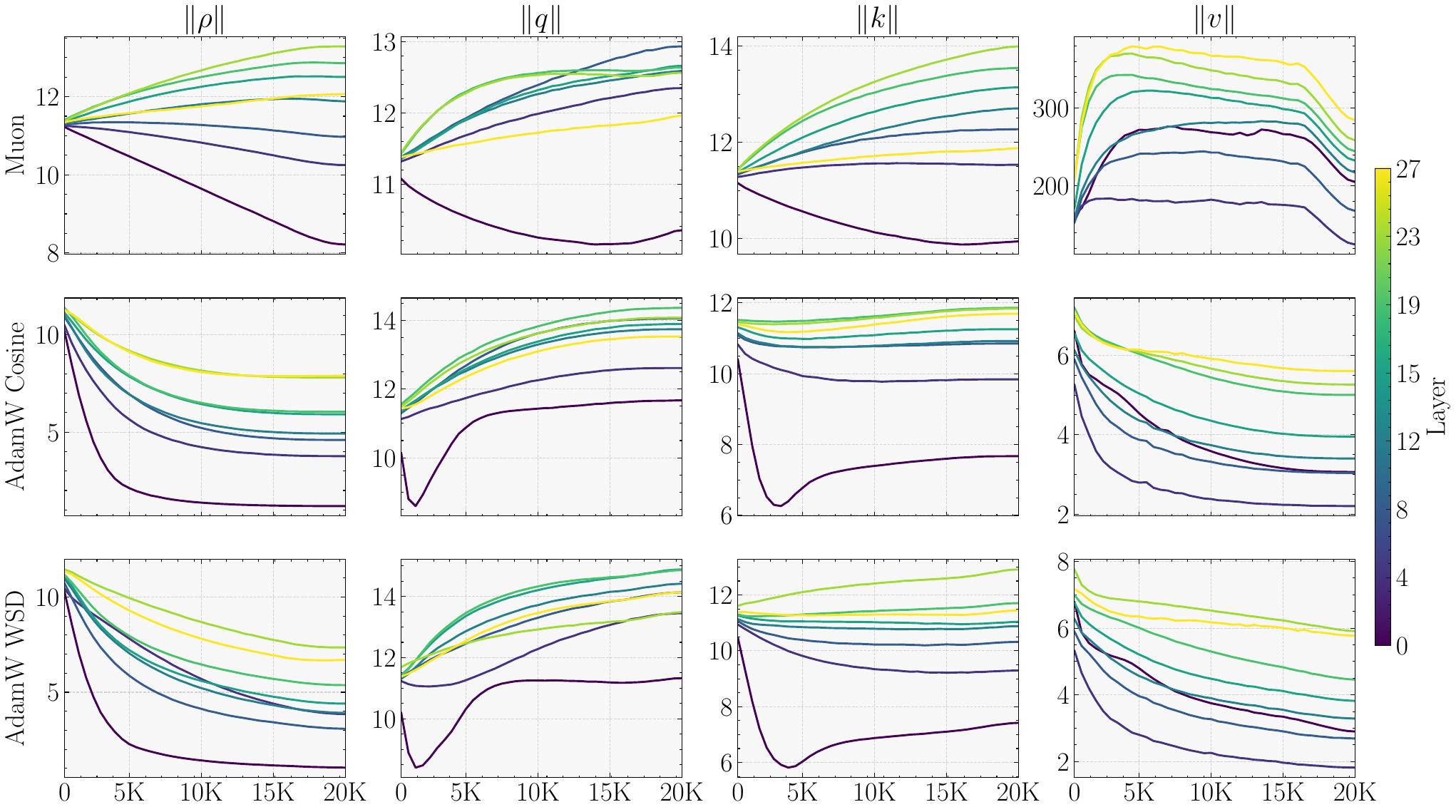}
        \caption{Evolution of the activation norms of $\bq$, $\bk$, $\bv$, and $\brho$ in the 0.6B Parallax models.}
        \label{fig:norm-evolve-qkvr}
    \end{subfigure}\\[0.6em]
    \begin{subfigure}{\linewidth}
        \centering
        \includegraphics[width=0.78\textwidth]{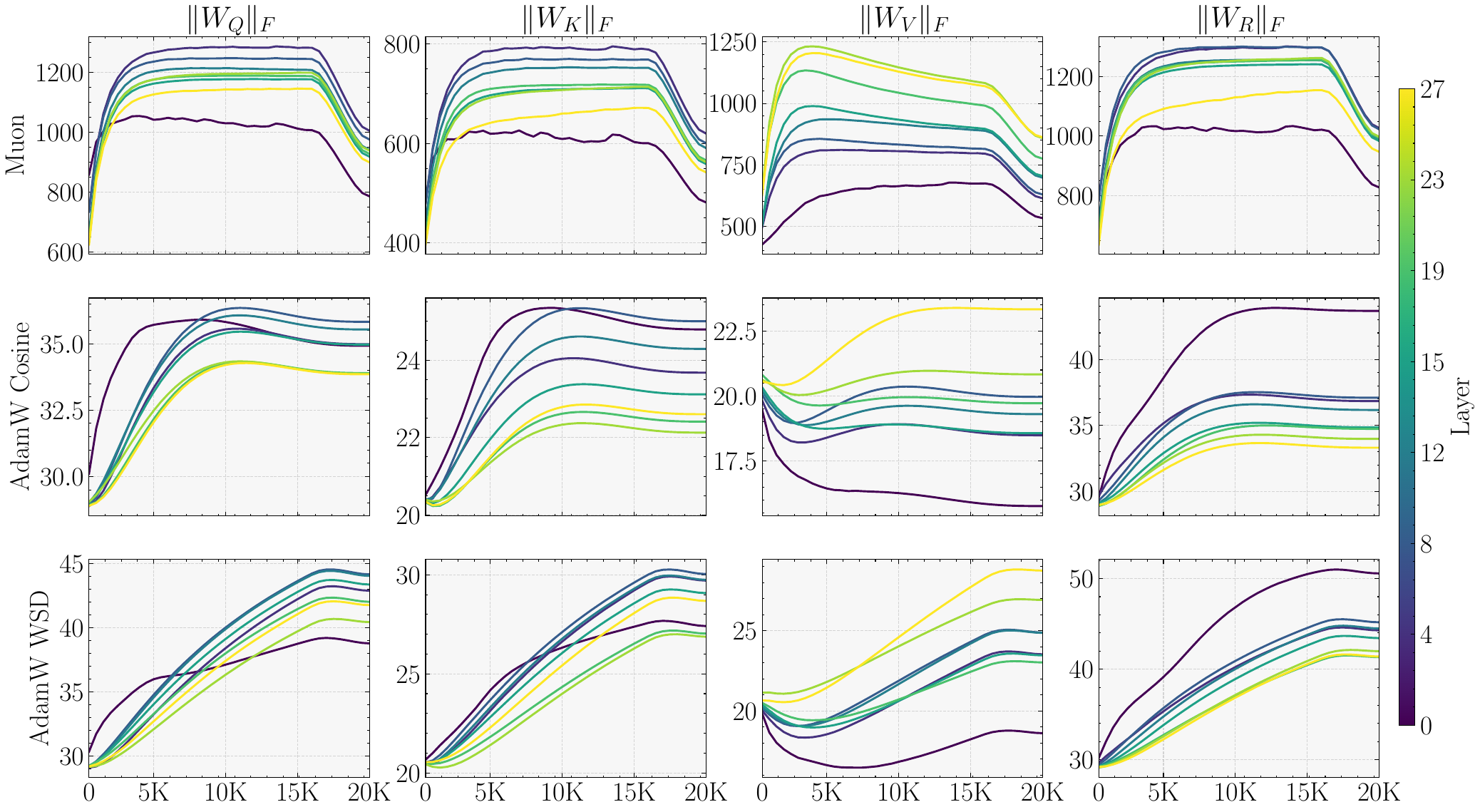}
        \caption{Evolution of the norms of $\bm W_Q$, $\bm W_K$, $\bm W_V$, $\bm W_O$, and $\bm W_R$ in the 0.6B Parallax models.}
        \label{fig:weight-norm-evolve}
    \end{subfigure}
    \caption{Training dynamics of activation norms (top) and projection weight norms (bottom).}
    \label{fig:norm-evolve-combined}
\end{figure}

\subsection{Training Dynamics}
\label{app:training_dynamics}

We track three metrics throughout the 0.6B Parallax pretraining run and report their evolution under different optimizer configurations.
Figure~\ref{fig:norm-evolve-qkvr} reports the layerwise activation norms of $\bq$, $\bk$, $\bv$, and $\brho$ vectors.
Figure~\ref{fig:weight-norm-evolve} reports the Frobenius norms of the projection weights $\bm W_Q$, $\bm W_K$, $\bm W_V$, $\bm W_O$, and $\bm W_R$.
Figure~\ref{fig:corr-ratio-combined} reports the correction to output ratio \texttt{COR} averaged over the sequence dimension.

Across all three diagnostics, Muon and AdamW exhibit qualitatively similar trajectories during the early phase, after which Muon continues to grow while AdamW saturates.
The activation norms $\|\bv\|$ and $\|\brho\|$ show the largest separation between optimizers.
The \texttt{COR} curves confirm that the correction branch is opened progressively under Muon and reaches its highest values in the deepest layers, whereas it remains largely suppressed under AdamW throughout training.
Together, these dynamics provide a temporal view of the optimizer-architecture interaction analyzed in Section~\ref{sec:analysis}.

\subsection{Advantage Shrinkage During Decay}
\label{app:decay}

\begin{figure}
    \centering
    \begin{subfigure}[c]{0.3\textwidth}
        \centering
        \includegraphics[width=\linewidth]{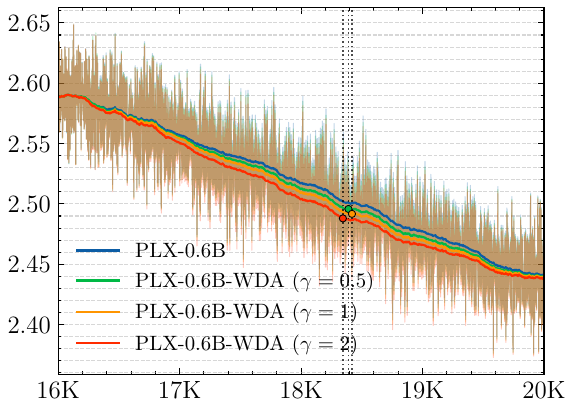}
        \caption{Training loss with WDA.}
        \label{fig:wda}
    \end{subfigure}
    \begin{subfigure}[c]{0.3\textwidth}
        \centering
        \includegraphics[width=\linewidth]{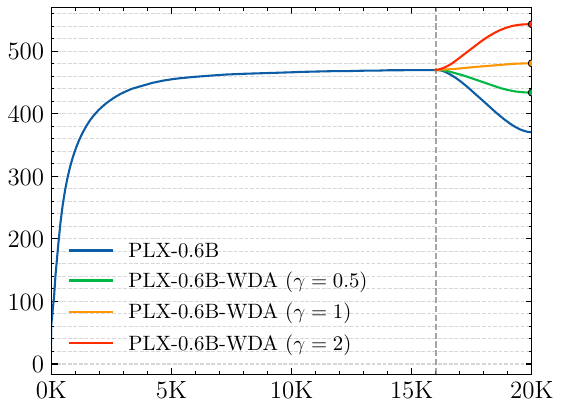}
        \caption{$\|\bm{W}_{R}\|_F$ evolution.}
        \label{fig:wda-norm}
    \end{subfigure}
    \caption{
        Figure~\ref{fig:wda} shows the training loss of Parallax with WDA during the decay stage. The turnaround point where the WDA variants start to lose their advantage is annotated.
        Figure~\ref{fig:wda-norm} shows the $\|\bm{W}_{R}\|_F$ evolution with WDA for the layer 18. WDA effectively mitigates the weight norm shrinkage.
    }
\end{figure}

Figure~\ref{fig:loss-curve-combined} shows the training curves of the 0.6B and 1.7B models under Muon with the WSD schedule, where the advantage of Parallax over the Transformer baseline is most pronounced.
However, the advantage of Parallax shrinks during the final linear decay phase of the WSD schedule.
The training dynamics in Figure~\ref{fig:weight-norm-evolve} shows that the weight norms shrink throughout the decay phase, which may partially explain the shrinkage of the Parallax advantage.

\paragraph{Weight decay annealing.}To investigate if the shrinkage of the Parallax advantage is related to the weight norm shrinkage, we run an additional ablation with weight decay annealing (WDA).
Concretely, let $t \in [0, 1]$ denote the fractional progress through the decay stage. WDA replaces the constant weight decay coefficient $\lambda$ with a schedule
\begin{align}
\lambda(t) = \lambda \cdot (1 - t)^\gamma,
\end{align}
where $\gamma \geq 0$ controls how aggressively the weight decay is annealed.
The choice $\gamma = 0$ recovers the standard WSD recipe.
For $\gamma > 0$, weight decay decreases over the course of the decay stage, with $\gamma = 1$ giving linear annealing and $\gamma = 2$ giving quadratic annealing that suppresses weight decay more aggressively toward the end of training.

\paragraph{Results.}
We apply WDA at 0.6B scale with $\gamma \in \{0.5, 1, 2\}$ and otherwise identical hyperparameters to the standard Muon with WSD run.
Figure~\ref{fig:wda-norm} confirms that WDA effectively mitigates the weight norm shrinkage, with larger $\gamma$ giving higher final norms of $\bm W_R$.
Figure~\ref{fig:wda} shows the training loss across the full decay window.
All three WDA variants reach lower final training loss than the standard recipe, and the improvement is monotone in $\gamma$ over the range tested where the largest annealing strength $\gamma = 2$ gives the largest gain.

The curves diverge gradually, with the gap widening throughout the first half of the decay stage.
However, the gap shrinks during the second half of the decay stage, with the final loss values converging to a narrow range.
We annotate the turnaround point where the WDA variants start to lose their advantage in the figure, which occurs at similar step counts across all three WDA variants.

We speculate that this late convergence is a consequence of how WDA affects the effective step size in parameter space.
The relative magnitude of a parameter update can be measured by the size of the step relative to the scale of the parameter itself:
\begin{align}
    \Delta_t = \frac{\|\eta_t \nabla_{\bm{W}_t}\mathcal L_t\|_F}{\|\bm{W}_t\|_F}.
\end{align}
Under standard WSD, both the numerator and the denominator shrink together over the decay stage, partially preserving the relative step size.
The shrinkage of denominator $\|\bm{W}_t\|_F$ is held by WDA, so as $\eta_t \to 0$ the relative step size collapses faster than under WSD.
WDA thus reduces the effective progress late in training.

\paragraph{Implications.}
WDA produces a clear gain in training performance, confirming that weight norm shrinkage is a real, mechanistic contributor to the decay stage advantage erosion, not an artifact of measurement.
This is a preliminary result that we report as evidence that the Muon with WSD recipe is not optimal for Parallax in its current form, and WDA only partially mitigates the issue.

\section{Parallax Score Visualizations}
\label{app:attention_maps}

\begin{figure}[h]
    \centering
    \begin{subfigure}{\linewidth}
        \centering
        \includegraphics[width=0.9\textwidth]{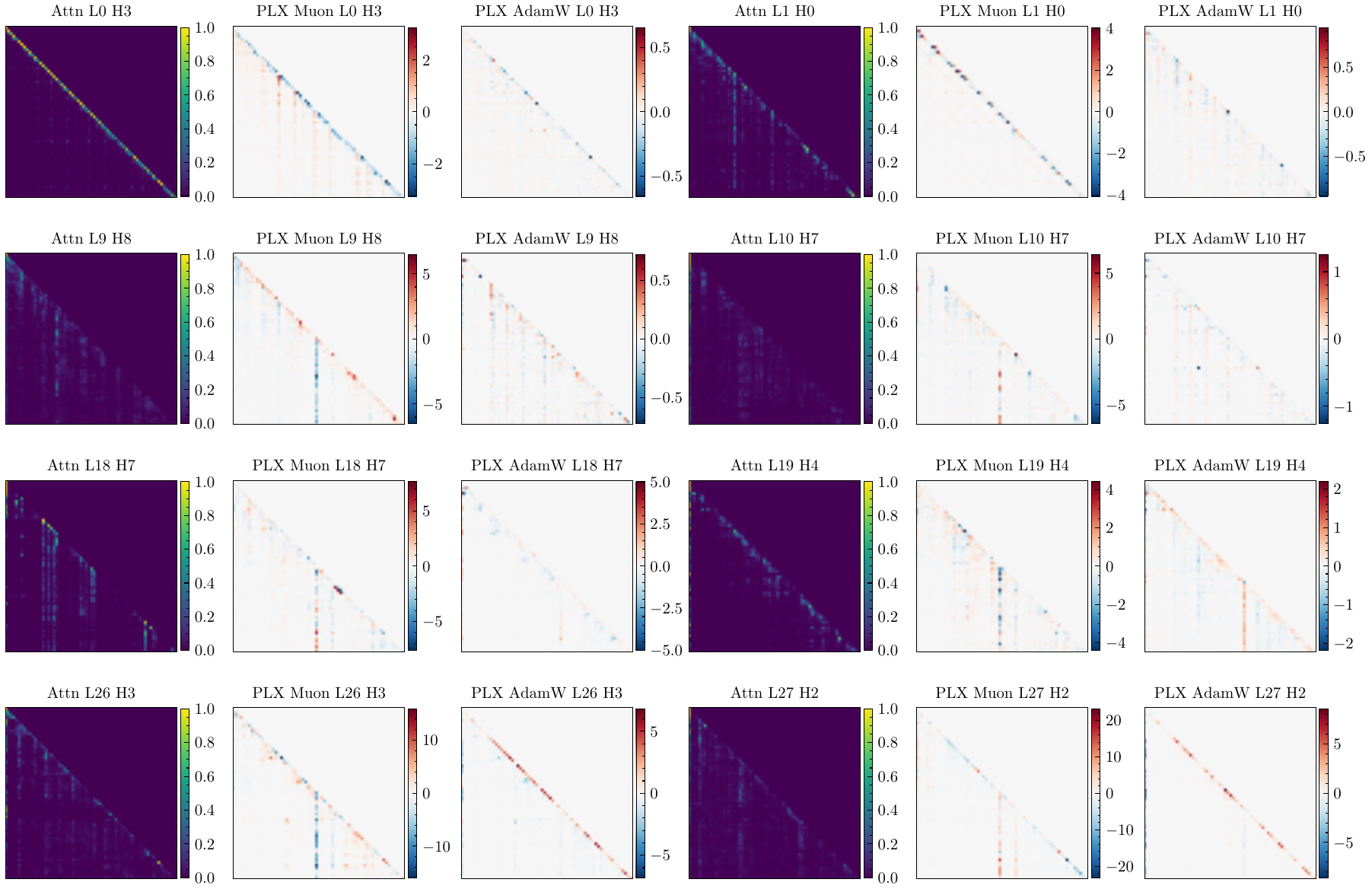}
        \caption{Score maps top left.}
        \label{fig:attn-score-top-left}
    \end{subfigure}\\[0.6em]
    \begin{subfigure}{\linewidth}
        \centering
        \includegraphics[width=0.9\textwidth]{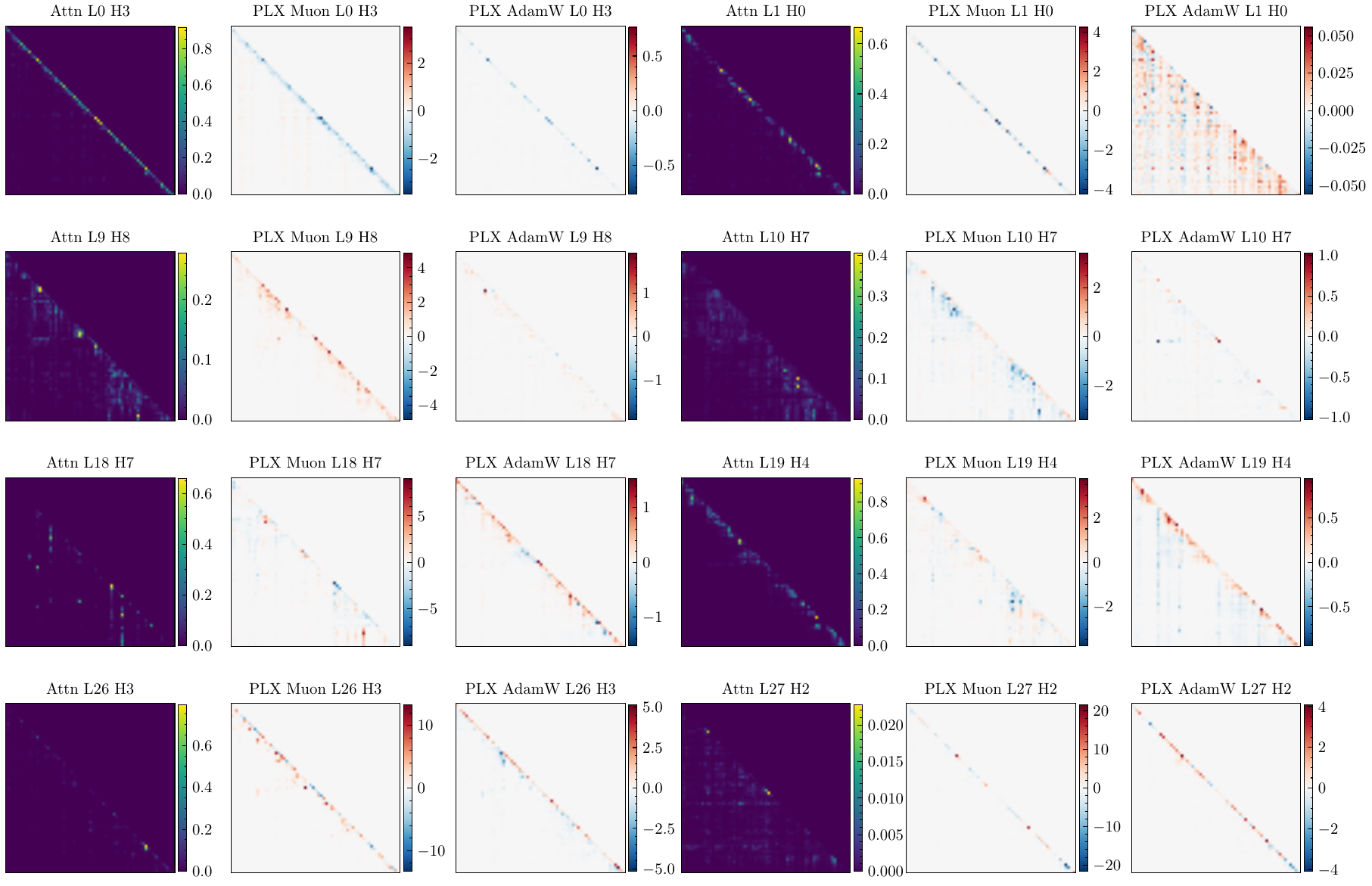}
        \caption{Score maps bottom right.}
        \label{fig:attn-score-bottom-right}
    \end{subfigure}
    \caption{Score maps of the Transformer and Parallax: top-left corner (top) and bottom-right corner (bottom).}
    \label{fig:attn-score-combined}
\end{figure}

To complement the aggregate score statistics in Section~\ref{sec:analysis}, we visualize the attention score maps of the Transformer baseline and Parallax.
Figure~\ref{fig:attn-score-top-left} shows the top left corner of the attention map, while Figure~\ref{fig:attn-score-bottom-right} shows the bottom right corner.
Each block contains $64\times 64$ tokens, and the input sequence is sampled from the pretraining data with sequence length $1024$.
The visualization of Parallax AdamW uses the WSD scheduler.

\end{document}